\documentclass[11pt,a4paper]{article}%
\usepackage[T1]{fontenc}
\usepackage{libertine}
\usepackage[table]{xcolor}
\usepackage[american]{babel}
\usepackage{xfrac}
\usepackage{bbm}
\usepackage{authblk}
\usepackage{amsthm,amssymb}
\usepackage{bm}
\newtheorem{proposition}{Proposition}[section]

\usepackage[format=plain,font=it]{caption}
\usepackage{algorithm}
\usepackage{algpseudocode}
\usepackage{pdflscape}

\usepackage{geometry}
\usepackage[round]{natbib} % has a nice set of citation styles and commands
\usepackage{mathtools} % amsmath with fixes and additions
\usepackage{booktabs} % commands to create good-looking tables
\usepackage{tikz} % nice language for creating drawings and diagrams
\usepackage{multirow}

\usepackage{float}
\usepackage{subfig}
\usepackage{amsmath}
\usepackage{amsfonts}
\usepackage{comment}
\usepackage{dsfont}
\newtheorem*{remark}{Remark}
\newtheorem*{definition}{Definition}
\newcommand{\indep}{\perp \!\!\! \perp}
\usepackage{array}
\newcolumntype{C}[1]{>{\centering\arraybackslash}m{#1}}
\newcolumntype{R}[1]{>{\raggedleft\arraybackslash}m{#1}}
\newcolumntype{L}[1]{>{\raggedright\arraybackslash}m{#1}}

\title{Sequential Conditional Transport on Probabilistic Graphs\\for Interpretable Counterfactual Fairness\thanks{Agathe Fernandes Machado acknowledges that the project leading to this publication has received funding from OBVIA. Arthur Charpentier acknowledges funding from the SCOR Foundation for Science and the National Sciences and Engineering Research Council (NSERC) for funding (RGPIN-2019-07077). Ewen Gallic acknowledges funding from the French government under the ``France 2030'' investment plan managed by the French National Research Agency (reference: ANR-17-EURE-0020) and from Excellence Initiative of Aix-Marseille University -- A*MIDEX.\\Replication codes and companion e-book: \href{https://github.com/fer-agathe/sequential_transport}{https://github.com/fer-agathe/sequential\_transport}}}

\usepackage{graphicx,pstricks,psfrag}
\definecolor{bleu}{RGB}{0,101,189}
\definecolor{vert}{HTML}{004D40}
\definecolor{rose}{HTML}{D81B60}
\definecolor{bleuTOL}{HTML}{332288}
\definecolor{wongBlack}{RGB}{0,0,0}
\definecolor{wongGold}{RGB}{230, 159, 0}
\definecolor{wongLightBlue}{RGB}{86, 180, 233}
\definecolor{wongGreen}{RGB}{0, 158, 115}
\definecolor{wongYellow}{RGB}{240, 228, 66}
\definecolor{wongBlue}{RGB}{0, 114, 178}
\definecolor{wongOrange}{RGB}{213, 94, 0}
\definecolor{wongPurple}{RGB}{204, 121, 167}
\definecolor{colUncalibrated}{RGB}{191, 191, 191}
\definecolor{colRecalibrated}{RGB}{197, 214, 231}

\definecolor{bleuTOL}{HTML}{332288}
\definecolor{vertTOL}{HTML}{117733}
\definecolor{vertClairTOL}{HTML}{44AA99}
\definecolor{bleuClairTOL}{HTML}{88CCEE}
\definecolor{sableTOL}{HTML}{DDCC77}
\definecolor{parmeTOL}{HTML}{CC6677}
\definecolor{magentaTOL}{HTML}{AA4499}
\definecolor{roseTOL}{HTML}{882255}

\definecolor{wongPurple}{RGB}{204, 121, 167}
\definecolor{wongLightBlue}{RGB}{86, 180, 233}
\definecolor{gris}{HTML}{A9A9A9}

\usepackage{hyperref}
\hypersetup{
    plainpages=false,
    bookmarksopen,
    bookmarksnumbered,
    unicode=true,          % non-Latin characters in Acrobat’s bookmarks
    pdftoolbar=true,        % show Acrobat’s toolbar?
    pdfmenubar=true,        % show Acrobat’s menu?
    pdffitwindow=false,     % window fit to page when opened
    pdfstartview={FitH},    % fits the width of the page to the window
    pdftitle={Sequential Conditional Transport on Probabilistic Graphs for Interpretable Counterfactual Fairness},    % title
    pdfauthor={Fernandes Machado, Agathe and Charpentier, Arthur and Gallic, Ewen},     % author
    pdfkeywords={fairness, individual fairness, knothe's rearrangement, optimal transport, sequential transport, causality, DAG}, % list of keywords
    pdfnewwindow=true,      % links in new PDF window
    colorlinks=true,       % false: boxed links; true: colored links
    linkcolor=black,          % color of internal links (change box color with linkbordercolor)
    citecolor={black},        % color of links to bibliography
    filecolor={rose},      % color of file links
    urlcolor={wongBlue},           % color of external links
}

\usepackage{doi}
\author[1]{Agathe~\textsc{Fernandes~Machado}\thanks{Corresponding author: \href{mailto:fernandes_machado.agathe@courrier.uqam.ca}{fernandes\_machado.agathe@courrier.uqam.ca}}}
\author[1]{Arthur~\textsc{Charpentier}}
\author[2,3]{Ewen~\textsc{Gallic}}

\affil[1]{%
    \footnotesize Département de Mathématiques\\
    Université du Québec à Montréal\\
    Montréal, Québec, Canada
}
\affil[2]{%
    \footnotesize Aix Marseille Univ, CNRS, AMSE\\
    Marseille, France
}
\affil[3]{%
    \footnotesize CNRS - Université de Montréal CRM -- CNRS
}
\usepackage[misc]{ifsym}
\makeatletter
\def\@fnsymbol#1{%
   \ifcase#1\or
   \TextOrMath ~ \dagger\or
   \TextOrMath {\footnotesize\Letter} \dagger\or
   \TextOrMath \textdaggerdbl \ddagger \or
   \TextOrMath \textsection  \mathsection\or
   \TextOrMath \textparagraph \mathparagraph\or
   \TextOrMath \textbardbl \|\or
   \TextOrMath {\textdagger\textdagger}{\dagger\dagger}\or
   \TextOrMath {\textdaggerdbl\textdaggerdbl}{\ddagger\ddagger}\else
   \@ctrerr \fi
}
\makeatother

\usepackage{fancyhdr} %pour modification des en-tête et pieds de page
\newcommand{\authornames}{\footnotesize\textsc{Fernandes Machado, Charpentier, Gallic}}
\usepackage{etoolbox}
\makeatletter
\patchcmd{\NAT@test}{\else \NAT@nm}{\else \NAT@nmfmt{\NAT@nm}}{}{}

\DeclareRobustCommand\citepos
  {\begingroup
   \let\NAT@nmfmt\NAT@posfmt% ...except with a different name format
   \NAT@swafalse\let\NAT@ctype\z@\NAT@partrue
   \@ifstar{\NAT@fulltrue\NAT@citetp}{\NAT@fullfalse\NAT@citetp}}

\let\NAT@orig@nmfmt\NAT@nmfmt
\def\NAT@posfmt#1{\NAT@orig@nmfmt{#1's}}
\makeatother

\begin{document}

\maketitle
\thispagestyle{empty}%no header/footer on the first page

\begin{center}
    This works appears in the\\
    Proceedings of the AAAI Conference on Artificial Intelligence, 39(18), 19358-19366.\\
doi:
\href{https://doi.org/10.1609/aaai.v39i18.34131}{10.1609/aaai.v39i18.34131}
\end{center}

\begin{abstract}
In this paper, we link two existing approaches to derive counterfactuals: adaptations based on a causal graph, and optimal transport. We extend  ``Knothe's rearrangement'' and ``triangular transport'' to probabilistic graphical models, and use this counterfactual approach, referred to as sequential transport, to discuss fairness at the individual level. After establishing the theoretical foundations of the proposed method, we demonstrate its application through numerical experiments on both synthetic and real datasets.
\end{abstract}

\section{Introduction}\label{sec:intro}

Most applications concerning discrimination and fairness are based on ``group fairness'' concepts (as introduced in \citet{hardt2016equality,kearns2019ethical}, or \citet{barocas2023fairness}). However, in many cases, fairness should be addressed at the individual level rather than globally. As claimed in \citet{dwork2012fairness}, ``{\em we capture fairness by the principle that any two individuals who are similar with respect to a particular task should
be classified similarly}.''  
The concept of ``counterfactual fairness'' was formalized in \citet{Kusner17}, addressing questions such as ``{\em had the protected attributes of the individual been different, other things being equal, would the decision had remain the same?}'' Such a statement has clear connections with causal inference, as discussed in \citet{pearl2018book}. Formally, consider observations \(\{s_i,\boldsymbol{x}_i,y_i\}\), where \(s\) is a binary protected attribute (e.g., \(s\in\{0,1\}\)), and \(\boldsymbol{x}\) is a collection of legitimate features (possibly correlated with \(s\)). The model output is \(y\), which is analyzed to address ``algorithmic fairness'' issues. Following \citet{rubin2005causal}, let \(y^\star(s)\) denote the potential outcome of \(y\) if \(s\) is seen as a treatment. 
With these notations, counterfactual fairness is achieved for individual \((s,\boldsymbol{x})\) if the average ``treatment effect,'' conditional on \(\boldsymbol{x}\) (or ``CATE'') is zero, i.e., \(\mathbb{E}[Y^\star(1)-Y^\star(0)|\boldsymbol{X}=\boldsymbol{x}]=0\). This quantity could be termed ``{\em ceteris paribus} CATE'' since all \(\boldsymbol{x}\)'s are supposed to remain unchanged for both treated and non-treated.

Following \citet{kilbertus2017avoiding}, it is possible to suppose that the protected attribute \(s\) could actually affect some explanatory variables \(\boldsymbol{x}\) in a non-discriminatory way. In \citet{charpentier2023optimal}, the outcome \(y\) was ``having a surgical intervention'' during childbirth in the U.S., \(s\) was the mother's ethnic origin (``Black'' or not) and \(\boldsymbol{x}\) included factors such as ``weight of the baby at birth.'' If Black mothers undergo less surgery because they tend to have smaller babies, there is no discrimination {\em per se}. At the very least, it should be fair, when assessing whether hospitals have discriminatory policies, to account for that difference in baby weights. Such a variable is named ``resolving variable'' in \citet{kilbertus2017avoiding}. Using heuristic notations, the ``{\em ceteris paribus} CATE'' \(\mathbb{E}[Y^\star(1)|\boldsymbol{X}=\boldsymbol{x}]-\mathbb{E}[Y^\star(0)|\boldsymbol{X}=\boldsymbol{x}]\) should become a ``{\em mutatis mutandis} CATE.'' For some individual \((s=0,\boldsymbol{x})\), this indicator would be \(\mathbb{E}[Y^\star(1)|\boldsymbol{X}=\boldsymbol{x}^\star(1)]-\mathbb{E}[Y^\star(0)|\boldsymbol{X}=\boldsymbol{x}]\), as coined in \citet{charpentier2023optimal}, to quantify discrimination, where fictitious individual \((s=1,\boldsymbol{x}^\star(1))\) is a ``counter\-factual'' version of \((s=0,\boldsymbol{x})\).

Two recent approaches have been proposed to assess counterfactual fairness using this \textit{mutatis mutandis} approach. 
On the one hand, \citet{plevcko2020fair} and \citet{plevcko2021fairadapt} used causal graphs (DAGs) to construct counterfactuals and assess the counterfactual fairness of outcomes \(y\) based on variables \((s,\boldsymbol{x},y)\). In network flow terminology, \(s\) acts as a ``source'' (only outgoing flow, or no parents), while \(y\) is a ``sink'' (only incoming flow). On the other hand, \citet{black2020fliptest}, \citet{charpentier2023optimal} and \citet{de2024transport} used optimal transport (OT) to construct counterfactuals. Moreover, using counterfactual reasoning to achieve fair machine learning (ML) models has also been notably studied \citep{ma_2023_ACM, robertson2024}. For evaluation, while \citet{de2024transport} provided a theoretical framework, its implementation is challenging (except in the Gaussian case), and usually hard to interpret. Here, we combine the two approaches, using OT within a causal graph structure.
The idea is to adapt ``Knothe's rearrangement'' \cite{bonnotte2013knothe}, or ``triangular transport'' \cite{zech2022sparse1,zech2022sparse2}, to a general probabilistic graphical model on \((s,\boldsymbol{x},y)\), rather than a simplistic \(s\to x_1\to x_2\to\cdots\to x_d\to y\). The concept of ``conditional OT'' has been recently discussed in \citet{bunne2022supervised} and \citet{hosseini2023conditional}, but here, instead of learning the causal graph, we assume a known causal graph and use it to construct counterfactual versions of individuals \((s_i,\boldsymbol{x}_i,y_i)\) to address fairness issues. Additionally, since we use univariate (conditional) transport, standard classical properties of univariate transport facilitates explanations (non-decreasing mappings, and quantile based interpretations).

\paragraph{Main Contributions}
\begin{itemize}
\item We use multivariate transport theory for constructing counterfactuals, as suggested in \citet{de2024transport}, and connect it to quantile preservation on causal graphs from \citet{plevcko2020fair} to develop a sequential transport methodology that aligns with the underlying DAG of the data.
\item Sequential transport, using univariate transport maps, provides closed-form solutions for deriving counterfactuals. This allows for the development of a data-driven estimation procedure that can be applied to new out-of-samples observations without recalculating, unlike multivariate OT with non-Gaussian distributions.
\item The approach's applicability is demonstrated through numerical experiments on both synthetic data and case studies, highlighting the interpretable analysis of individual counterfactual fairness when using sequential transport.
\end{itemize}

Section~\ref{sec:causal} introduces various concepts used in probabilistic graphical models from a causal perspective. Section~\ref{sec:O:T}, revisits classical OT covering both univariate and multivariate cases. Sequential transport is covered in Section~\ref{sec:seq:transport}. Section~\ref{sec:fairness} discusses counterfactual fairness. Illustration with real data are provided in Section~\ref{sec:real-data}.

\section{Graphical Models and Causal Networks}\label{sec:causal}

\subsection{Probabilistic Graphical Models}

Following standard notations in probabilistic graphical models (see \citet{koller2009probabilistic} or \citet{barber2012bayesian}), given a random vector \(\boldsymbol{X}=(X_1,\cdots,X_d)\), consider a directed acyclic graph (DAG) \(\mathcal{G}=(V,E)\), where \(V=\{x_1,x_2,\cdots,x_d\}\) are the vertices (corresponding to each variable), and \(E\) are directed edges, such that \(x_i\to x_j\) means ``variable \(x_i\) causes variable \(x_j\),'' in the sense of \citet{susser1991cause}. The joint distribution of \(\boldsymbol{X}\) satisfies the (global) Markov property w.r.t. \(\mathcal{G}\):
\[
    {\displaystyle \mathbb{P}[x_{1},\cdots ,x_{d}]=\prod _{j=1}^{d}\mathbb{P}[x_{j}|{\text{parents}(x_{j})}]},
\]
where {\(\text{parents}(x_{i})\)} are nodes with edges directed towards \(x_{i}\), in \(\mathcal{G}\). \citet{watson2021local} suggested the  causal graph in Figure~\ref{fig:dag} for the German Credit dataset, where \(s\) is the ``sex'' (top left) and \(y\) is the ``default'' indicator (right). Observe that variables \(\boldsymbol{x}_j\) are here sorted.
As discussed in \citet{ahuja1993network}, such a causal graph imposes some ordering on variables. 
In this ``topological sorting,'' a vertex must be selected before its adjacent vertices, which is feasible because each edge is directed such that no cycle exists in the graph. 
In our analysis, we consider a network \(\mathcal{G}\) on variables \(\{s,\boldsymbol{x},y\}\) where \(s\) is the sensitive attribute, acting as a ``source'' (only outgoing flow, or no parents) while \(y\) is a ``sink'' (only incoming flow, i.e., \(y\notin \text{parents}(x_{i})\), \(\forall i\)). 

\begin{figure}[htb]
    \centering
    \begin{tikzpicture}[scale=.9]
    \useasboundingbox (2, 0.5) rectangle (9.5, 5);
    \node[fill=red!30] (sex) at (3.25,4.25) {sex $s$};
    \node[fill=yellow!60] (age) at  (2.5,1.5) {age $x_1$};
    \node[fill=yellow!60] (job) at  (2,3) {job $x_2$};
    \node[fill=yellow!60] (sav) at (4.5,3) {savings $x_3$};
    \node[fill=yellow!60] (hou) at (5.75,4.45)  {housing $x_4$};
    \node[fill=yellow!60] (cre) at (6,2)  {credit $x_5$};
    \node[fill=yellow!60] (dur) at (8,4) {duration $x_6$};
    \node[fill=yellow!60] (pur) at (8,2) {purpose $x_7$};
    \node[fill=blue!30] (y) at (9.5,3.25) {default $y$};
    
    \path[->, red][line width=1pt] (sex) edge[out=210, in=70] (job);
    \path[->, red][line width=1pt] (sex) edge[out=330, in=110] (sav);
    \path[->] (age) edge[out=75, in=270] (job);
    \path[->] (age) edge[out=10, in=200] (sav);
    \path[->] (age) edge[out=340, in=200] (cre);
    \path[->] (age) edge[out=320, in=210] (pur);
    \path[->] (job) edge[out=340, in=170] (hou);
    \path[->] (job) edge[out=20, in=160] (sav);
    \path[->] (job) edge[out=320, in=170] (cre);
    \path[->] (hou) edge[out=270, in=90] (cre);
    \path[->] (hou) edge[out=320, in=110] (pur);
    \path[->] (cre) edge[out=40, in=220] (dur);
    \path[->] (cre) edge[out=0, in=180] (pur);
    \path[->] (sav) edge[out=40, in=240] (hou);
    \path[->] (sav) edge[out=20, in=170] (dur);
    \path[->] (dur) edge[out=320, in=60] (pur);
    \path[->] (pur) edge[out=30, in=230] (y);
    \end{tikzpicture}
    \caption{Causal graph in the German Credit dataset from \citet{watson2021local}, or DAG.}
    \label{fig:dag}
\end{figure}

\subsection{Causal Networks and Linear Structural Models}

\citet{wright1921correlation,wright1934method} used {directed graphs} to represent probabilistic cause-and-effect relationships among a set of variables and developed path diagrams and path analysis. Simple causal networks can be visualized on top of Figure~\ref{fig:SCM-0}. On the left is a simple model where the ``cause'' \(C\)  directly causes (\(\to\)) the ``effect'' \(E\). On the right, a ``mediator'' \(X\) is added. There is still the direct impact of \(C\) on \(E\) (\(C\to E\)), but there is also a mediated indirect impact (\(C\to X\to E\)).

\subsubsection{Intervention in a Linear Structural Model} In a simple causal graph, with two nodes, \(C\) (the cause) and \(E\) (the effect), the causal graph is \(C \rightarrow E\), and the mathematical interpretation can be summarized in two (linear) assignments:%
\begin{equation}\label{fig:SCM-Gauss}
\begin{cases}
    C = a_c+ U_C\\
    E = a_e+b_e C+U_E,
\end{cases}     
 \end{equation}
 where \(U_C\) and \(U_E\) are independent Gaussian random variables. That causal graph can be visualized in Figure~\ref{fig:SCM-0}, and its corresponding structural causal model (SCM) described in Equation~\ref{fig:SCM-Gauss} illustrates the causal relationships between variables, as in \citet{Pearl2000TheLO}.
Suppose here that \(C\) is a binary variable, taking values in \(\{c_0,c_1\}\).
 Given an observation \((c_0,e)\), the ``counterfactual outcome'' if the cause had been set to \(c_1\) (corresponding to the intervention in Figure~\ref{fig:SCM-1}), would be \(e + b_e (c_1-c_0)\). Following \citet{pearl2009causality}, one can also introduce the ``twin network'' corresponding to a mirrored version of the initial causal graph in the counterfactual world. \citet{plevcko2020fair} coined this approach ``fair-twin projection'' when \(C\) is a binary sensitive attribute.

\begin{comment}
\subsubsection{Using a Twin Network} Following \citet{pearl2009causality}, one can introduce the ``twin network'' shown in Figure~\ref{fig:SCM-2}, with the initial causal graph at the bottom, and a mirrored version (the twin causal graph) in the counterfactual world at the top. The idea is that the causal mapping \(C\to E\) remains unchanged (we simply consider a different input value). \citet{plevcko2020fair} coined this approach ``fair-twin projection'' when \(C\) is a binary sensitive attribute, and the goal is to assess fairness in the effect \(E\).
 In Figures~\ref{fig:SCM-0}--\ref{fig:SCM-2}, a mediator variable \(X\) is added on the right. Given the observation \((c_0,x,e)\), the counterfactual outcome associated with cause \(c_1\) is \(e^\star\). By following the DAG in Figure~\ref{fig:SCM-1}, we can determine the effects of the do-intervention \(\text{do}(C=c_1)\) through the pathways of \(X\) and \(E\), using the observed data to estimate the error terms. 
\end{comment}

\begin{figure}[htb]
    \centering
    \begin{tabular}{ccc}
    \tikz{
    \node[fill=yellow!60] (x) at (0,0) {$C$};
    \node[fill=yellow!60] (z) at (2,0) {$E$};
    \node[fill=yellow!20] (ux) at (0,1) {$u_C$};
    \node[fill=yellow!20] (uz) at (2,1) {$u_E$};
    \path[->, black] (x) edge (z);
    \path[->, black] (ux) edge (x);
    \path[->, black] (uz) edge (z);
    \path[->, white, bend right=60] (x) edge (z);
}  
& & 
\hspace{-1cm}\tikz{
    \node[fill=yellow!60] (x) at (0,0) {$C$};
    \node[fill=yellow!60] (y) at (1,0) {$X$};
    \node[fill=yellow!60] (z) at (2,0) {$E$};
    \node[fill=yellow!20] (ux) at (0,1) {$u_C$};
    \node[fill=yellow!20] (uy) at (1,1) {$u_X$};
    \node[fill=yellow!20] (uz) at (2,1) {$u_E$};
    \path[->, black] (x) edge (y);
    \path[->, black] (y) edge (z);
    \path[->, black, bend right=60] (x) edge (z);
    \path[->, black] (ux) edge (x);
    \path[->, black] (uy) edge (y);
    \path[->, black] (uz) edge (z);
}       \\

$\displaystyle{\begin{cases}
C = a_c+U_C\\
E = a_e+b_e C+U_E
\end{cases}}$ &&$
\hspace{-1cm}\displaystyle{\begin{cases}
C = a_c+U_C\\
X = a_x+b_x C+U_X\\
E = a_e+b_e C+\gamma_e X+U_E \\
\end{cases}}$
    \end{tabular}
    \caption{Linear Structural Causal Model -- observation.}
    \label{fig:SCM-0}
\end{figure}

\begin{figure}[htb]
    \centering
    \begin{tabular}{ccc}
 %   (a)  observation \phantom{$\displaystyle\int$} & & (a) observation\\
    \tikz{
    \node[shape = circle, inner sep = 3pt,fill=red!50]  (x) at (0,0) {$c$};
     \node[fill=yellow!60] (z) at (2,0) {$E^\star$};
    \node[fill=yellow!20] (uz) at (2,1) {$u_E$};
    \path[->, black] (x) edge (z);
    \path[->, black] (uz) edge (z);
     \path[->, white, bend right=60] (x) edge (z);
}
& & 
\hspace{-1cm}\tikz{
    \node[shape = circle, inner sep = 3pt,fill=red!50]  (x) at (0,0) {$c$};
    \node[fill=yellow!60] (y) at (1,0) {$X^\star$};
    \node[fill=yellow!60] (z) at (2,0) {$E^\star$};
    \node[fill=yellow!20] (uy) at (1,1) {$u_X$};
    \node[fill=yellow!20] (uz) at (2,1) {$u_E$};
    \path[->, black] (x) edge (y);
    \path[->, black] (y) edge (z);
    \path[->, black, bend right=60] (x) edge (z);
    \path[->, black] (uy) edge (y);
    \path[->, black] (uz) edge (z);
}      \\

 $\displaystyle{\begin{cases}
C = c~~(\text{or do}(C=c))\\
E_c^\star = a_e+b_e c +U_E \\
\end{cases}}$ &&
\hspace{-1cm}$\displaystyle{\begin{cases}
C = c~~(\text{or do}(C=c))\\
X_c^\star = a_x+b_x c+U_X\\
E_c^\star = a_e\!+\!b_e c\!+\!\gamma_e X_c^\star\!+\!U_E \\
\end{cases}}$
\end{tabular}
    \caption{Linear Structural Causal Model -- intervention.}
    \label{fig:SCM-1}
\end{figure}

\subsection{Non-Linear Structural Models}\label{sub:sec:NPSEM}

\subsubsection{Presentation of the Model}

More generally, consider a non-Gaussian and nonlinear structural model, named ``non-parametric structural equation model'' (with independent errors) in \citet{Pearl2000TheLO},
 \begin{equation}\label{fig:SCM-general}
\begin{cases}
C = h_c(U_C)\\
E = h_e(C,U_E), \\
\end{cases}
 \end{equation}
 where \(u\mapsto h_c(\cdot, u)\) and \(u\mapsto h_e(\cdot, u)\) are strictly increasing in \(u\); \(U_C\) and \(U_E\) are independent, and,  without loss of generality, supposed to be uniform on \([0,1]\). For a rigorous mathematical framework for non-linear non-Gaussian structural causal models, see \citet{bongers2021foundations} or \citet{shpitser2022multivariate}. 

 \subsubsection{Connections With Conditional Quantiles}

 Consider now some general DAG, \(\mathcal{G}\), on \(\boldsymbol{X}=(X_1,\cdots,X_d)\), supposed to be absolutely continuous. With previous notations, \(X_i=h_i(\text{parents}(X_i),U_i)\), a.s., for all variables, representing the structural equations. 
 We can write this compactly as \(\boldsymbol{X}=h({\text{parents}}(\boldsymbol{X}),\boldsymbol{U})\), a.s., by considering \(h\) as a vector function. Solving the structural model means finding a function \(g\) such that \(\boldsymbol{X}=g(\boldsymbol{U})\), a.s. 
 To illustrate, consider a specific \(i\), and \(X_i = h_i(\text{parents}(X_i),U_i)\). If \(\text{parents}(X_i)=\boldsymbol{x}\) is fixed, define \(h_{i|\boldsymbol{x}}(u)=h_i(\boldsymbol{x},u)\). Let \(U\) be a uniform random variable, and let \(F_{i|\boldsymbol{x}}\) be the cumulative distribution of \(h_{i|\boldsymbol{x}}(U)\),
\(
F_{i|\boldsymbol{x}}(x) = \mathbb{P}[h_{i|\boldsymbol{x}}(U)\leq x].
\)
Since \(X_i\) is absolutely continuous, \(F_{i|\boldsymbol{x}}\) is invertible, and \(F_{i|\boldsymbol{x}}^{-1}\) is a conditional quantile function (conditional on \(\text{parents}(X_i)=\boldsymbol{x}\)). Let \(V=F_{i|\boldsymbol{x}}(h_{i|\boldsymbol{x}}(U))\), then \(X_i = F_{i|\boldsymbol{x}}^{-1}(V)\) and \(V\) is uniformly distributed on \([0,1]\). This means that \(x_i = h_{i|\boldsymbol{x}}(u_i) \) corresponds to the quantile of variable \(X_i\), conditional on the values of its parents, \(\text{parents}(X_i)\), with probability level \(u_i\). In the observational world, \(u_i\) represents the (conditional) probability level associated with observation \(x_i\), and its counterfactual counterpart is \(x_i^\star\) corresponding to the (conditional) quantile associated with the same probability level \(u_i\).

This representation has been used in \citet{plevcko2020fair} and \citet{plevcko2021fairadapt}, where \(X_i = F_{i|\boldsymbol{x}}^{-1}(V)\) is simply the probabilistic representation of ``quantile regression,'' as introduced by \citet{koenker1978regression} (and further studied in \citet{koenker2005quantile} and \citet{regression2017handbook}). 
This could be extended to ``quantile regression forests,'' as in \citet{meinshausen2006quantile}, or any kind of ML model, as \citet{Cannon2018} or \citet{NEURIPS2022}. Observe that \citet{ma2006quantile} considered some close ``recursive structural equation models,'' characterized by a system of equations where each endogenous variable is regressed on other endogenous and exogenous variables in a hierarchical manner. They used some sequential quantile regression approach to solve those recursive SEMs. 
An alternative we consider here is to use the connection between quantiles and OT (discussed in \citet{chernozhukov2013inference} or \citet{hallin2024multivariate}) to define some ``conditional transport'' that relates to those conditional quantiles.

\section{Optimal Transport}\label{sec:O:T}

Given two metric spaces \(\mathcal{X}_0\) and \(\mathcal{X}_1\), consider a measurable map \(T:\mathcal{X}_0\to\mathcal{X}_1\) and a measure \(\mu_0\) on \(\mathcal{X}_0\).
The {push-forward} of \(\mu_0\) by \(T\) is the measure \(\mu_1 = T_{\#}\mu_0\) on \(\mathcal{X}_1\) defined by \(T_{\#}\mu_0(B)=\mu_0\big(T^{-1}({B})\big)\), \(\forall {B}\subset\mathcal{X}_1\).
For all measurable and bounded \(\varphi:\mathcal{X}_1\to\mathbb{R}\),
\[\int_{\mathcal{X}_1}\varphi (x_1)T_{\#}\mu_0(\mathrm{d}x_1) = 
\int_{\mathcal{X}_0}\varphi\big(T(x_0)\big)\mu_0(\mathrm{d}x_0).
\]
For our applications, if we consider measures \(\mathcal{X}_0=\mathcal{X}_1\) as a compact subset of \(\mathbb{R}^d\), then there exists \(T\) such that \(\mu_1 = T_{\#}\mu_0\), when \(\mu_0\) and \(\mu_1\) are two measures, and \(\mu_0\) is atomless, as shown in \citet{villani2003optimal} and \citet{santambrogio2015optimal}. In that case, and if we further suppose that measures \(\mu_0\) and \(\mu_1\) are absolutely continuous, with densities \(f_0\) and \(f_1\) (w.r.t. Lebesgue measure), a classical change of variable expression can be derived. Specifically, the previous integral
\[
\int_{\mathcal{X}_1}\varphi (\boldsymbol{x}_1)f_1(\boldsymbol{x}_1)\mathrm{d}\boldsymbol{x}_1
\]
is simply (if \(\nabla T\) is the Jacobian matrix of mapping \(T\)):
\[   \int_{\mathcal{X}_0}\varphi\big(T(\boldsymbol{x}_0)\big)~ \underbrace{f_1(T(\boldsymbol{x}_0))\det \nabla T(\boldsymbol{x}_0)}_{=f_0(\boldsymbol{x}_0)}~ \mathrm{d}\boldsymbol{x}_0.
\]

Out of those mappings from \(\mu_0\) to \(\mu_1\), we can be interested in ``optimal'' mappings, satisfying Monge problem, from \citet{monge1781memoire}, i.e., solutions of 
\[
\inf_{T_{\#}\mu_0=\mu_1} \int_{\mathcal{X}_0} c\big(\boldsymbol{x}_0,T(\boldsymbol{x}_0)\big)\mu_0(\mathrm{d}\boldsymbol{x}_0),
\]
for some positive ground cost function \(c:\mathcal{X}_0\times\mathcal{X}_1\to\mathbb{R}_+\). 

In general settings, however, such a deterministic mapping \(T\) between probability distributions may not exist (in particular if \(\mu_0\) and \(\mu_1\) are not absolutely continuous, with respect to Lebesgue measure). This limitation motivates the Kantorovich relaxation of Monge's problem, as considered in \citet{kantorovich1942translocation},
\[
\inf_{\pi\in\Pi(\mu_0,\mu_1)} \int_{\mathcal{X}_0\times\mathcal{X}_1} c(\boldsymbol{x}_0,\boldsymbol{x}_1)\pi(\mathrm{d}\boldsymbol{x}_0,\mathrm{d}\boldsymbol{x}_1),
\]
with our cost function \(c\), where \({\displaystyle \Pi (\mu_0 ,\mu_1 )}\) is the set of all couplings of \(\mu_0\) and \(\mu_1\). This problem focuses on couplings rather than deterministic mappings It always admits solutions referred to as OT plans.

\subsection{Univariate Optimal Transport}

Suppose here that \(\mathcal{X}_0=\mathcal{X}_1\) is a compact subset of \(\mathbb{R}\). The optimal Monge map \(T^\star\) for some strictly convex cost \(c\) such that \(T^\star_{\#}\mu_0=\mu_1\) is \(T^\star=F_{1}^{-1}\circ F_{0}\), where \(F_i:\mathbb{R}\to[0,1]\) is the cumulative distribution function associated with \(\mu_i\), \(F_i(x)=\mu_i((-\infty,x])\), and \(F_i^{-1}\) is the generalized inverse (corresponding to the quantile function), \(F^{-1}_i(u) = \inf  \big\{x \in \mathbb{R}: F_i(x) \geq u \big\}\). Observe that \(T^\star\) is an {increasing mapping} (which is the univariate definition of being the gradient of a convex function, from \citet{brenier1991polar}). 
This is illustrated in Figure~\ref{fig:transport}, with a Gaussian case on the left (\(T^\star\) affine), and general densities on the right.

\begin{figure}[htb]
    \centering
    \includegraphics[width=.5\columnwidth]{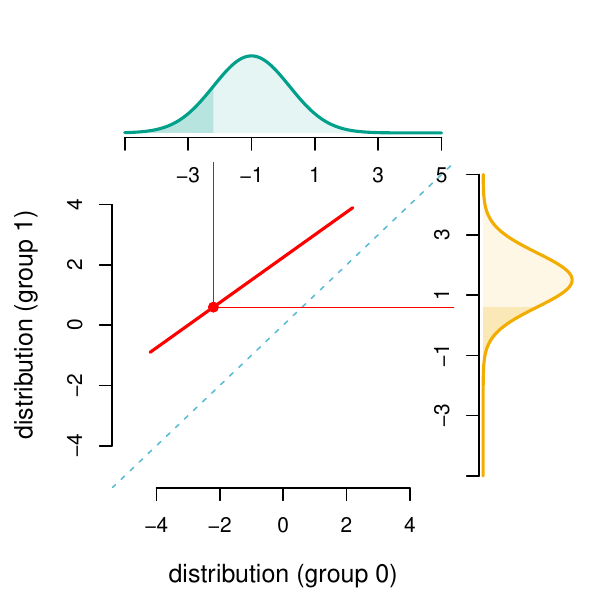}\includegraphics[width=.5\columnwidth]{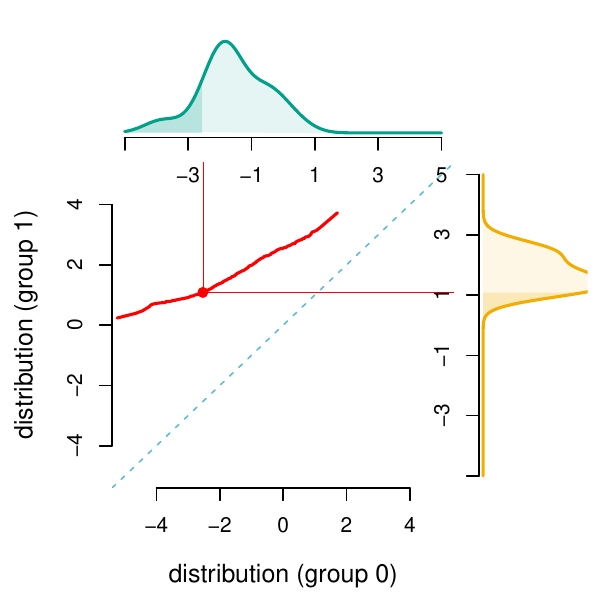}
    \caption{Univariate OT, with Gaussian distributions (left) and general marginal distributions (right). The transport curve (\(T^\star\)) is shown in red.}
    \label{fig:transport}
\end{figure}

\subsection{Multivariate Optimal Transport}

In a multivariate setting, when \(\mathcal{X}_0=\mathcal{X}_1\) is a compact subset of \(\mathbb{R} ^d\), from \citet{brenier1991polar}, with a quadratic cost, the optimal Monge map \(T^\star\) is unique, and it is the gradient of a convex mapping \(\psi:\mathbb{R} ^d\to\mathbb{R} ^d\), \(T^\star=\nabla \psi\). Therefore, its Jacobian matrix \(\nabla T^\star\) is nonnegative and symmetric. More generally, with strictly convex cost
in \(\mathbb{R}^d\times\mathbb{R}^d\), the Jacobian
matrix \(\nabla T^\star\), even if not necessarily nonnegative symmetric, is diagonalizable with nonnegative eigenvalues, as proved in \citet{cordero2004non} and \citet{ambrosio2005gradient}. Unfortunately, it is generally difficult to give an analytic expression for the optimal mapping \(T^\star\), unless additional assumptions are made, such as assuming that both distributions are Gaussian, as in Appendix~\ref{sec:appendix-gaussian}.

\section{Sequential Transport}\label{sec:seq:transport}

\subsection{Knothe-Rosenblatt Conditional Transport} \label{subsec:kr-def}

As explained in \citet{villani2003optimal,carlier2010knothe,bonnotte2013knothe}, the Knothe-Rosenblatt (KR) rearrangement is directly inspired by the Rosenblatt chain rule, from \citet{rosenblatt1952remarks}, and some extensions obtained on general measures by \citet{knothe1957contributions}. Using notations of Section 2.3 in \citet{santambrogio2015optimal}, 
let \(\mu_{0:d}\) denote the marginal \(d\)-th measure, \(\mu_{0:d-1|d}\) the conditional \(d-1\)-th measure (given \(x_{d}\)), \(\mu_{0:d-2|d-1,d}\) the conditional \(d-2\)-th measure (given \(x_{d-1}\) and \(x_{d}\)), etc. Suppose that the \( \mu_{0} \)-conditionals, corresponding to measures \( \mu_{0:d} \), \( \mu_{0:d-1|d} \), etc., are atomless (satisfied as soon as \( \mu_0 \) is absolutely continuous with respect to the Lebesgue measure). For the first two,
\begin{align*}
\mu_0(\mathbb{R}^{d-1}\!\times\mathrm{d}x_d) & = \mu_{0:d}(\mathrm{d}x_d) \\
\mu_0(\mathbb{R}^{d-2}\times\mathrm{d}x_{d-1}\times\mathrm{d}x_d) & \!=\! \mu_{0:d}(\mathrm{d}x_d)\mu_{0:d-1|d}(\mathrm{d}x_{d-1}|x_d) 
\end{align*}
and iterate. Define conditional (univariate) cumulative distribution functions:
\[
\begin{cases}
F_{0:d}(x_d)=\mu_{0:d}((-\infty,x_d])=\mu_0(\mathbb{R}^{d-1}\times (-\infty,x_d])\\
F_{0:d-1|d}(x_{d-1}|x_d)=\mu_{0:d-1|d}((-\infty,x_{d-1}]|x_d),
\end{cases}
\]
etc. And similarly for \(\mu_1\). For the first component, let \(T^\star_d\) denote the monotone nondecreasing map transporting from \(\mu_{0:d}\) to \(\mu_{1:d}\), defined as \(T^\star_{\overline{d}}(\cdot)=F_{1:d}^{-1}(F_{0:d}(\cdot))\). For the second component, let \(T^\star_{\overline{d-1}}(\cdot|x_d)\) denote the monotone nondecreasing map transporting from \(\mu_{0:d-1|d}(\cdot|x_d)\) to \(\mu_{1:d-1|d}(\cdot|T^\star_{\overline{d}}(x_d))\), \(T^\star_{\overline{d-1}}(\cdot|x_d)=F_{1:d-1|d}^{-1}(F_{0:d-1|d}(\cdot|x_d)|T^\star_{\overline{d}}(x_d))\). We can then repeat the construction, and finally, the KR rearrangement is
\[
T_{\overline{kr}}(x_1,\cdots,x_d) = 
\begin{pmatrix}
T^\star_{\overline1}(x_1|x_{2},\cdots,x_d)\\
T^\star_{\overline2}(x_2|x_{3},\cdots,x_d)\\
\vdots\\
T^\star_{\overline{d-1}}(x_{d-1}|x_d)\\
T^\star_{\overline d}(x_d)
\end{pmatrix}.
\]
As proved in \citet{santambrogio2015optimal} and \citet{carlier2010knothe}, \(T_{\overline{kr}}\) is a transportation map from \(\mu_0\) to \(\mu_1\), in the sense that \(\mu_1 = T_{\overline{kr}\#}\mu_0\). Following \citet{bogachev2005triangular} and \citet{backhoff2017causal}, \(T_{\overline{kr}}\) is the ``monotone upper triangular map'' uniquely defined when the \(\mu_1\)-conditionals are atomless for a chosen coordinate order. \citet{bogachev2005triangular} defined the ``monotone lower triangular map,''
\[
T_{\underline{kr}}(x_1,\cdots,x_d) = 
\begin{pmatrix}
    T^\star_{\underline1}(x_1)\\
T^\star_{\underline2}(x_2|x_{1})\\
\vdots\\
T^\star_{\underline{d-1}}(x_{d-1}|x_1,\cdots,x_{d-2})\\
T^\star_{\underline d}(x_{d}|x_1,\cdots,x_{d-1})
\end{pmatrix}.
\]
The map \(x_{i}\mapsto T^\star_{\underline i}(x_i|x_{1},\cdots,x_{i-1})\) is monotone (nondecreasing) for all \((x_{1},\cdots,x_{i-1})\in\mathbb{R}^{i-1}\).
Further, by construction, this KR transport map has a triangular Jacobian matrix \(\nabla T_{\underline{kr}}\) with nonnegative entries on its diagonal, making it suitable for various geometric applications. However, this mapping does not satisfy many properties; for example, it is not invariant under isometries of \(\mathbb{R}^d\) as mentioned in \citet{villani2009optimal}.
\citet{carlier2010knothe} proved that the KR transport maps could be seen as limits of quadratic OTs. A direct interpretation is that this iterative sequential transport can be seen as ``marginally optimal.'' Some explicit formulas can be obtained in the Gaussian case, as discussed in Appendix~\ref{sec:appendix-gaussian}.

\subsection{Sequential Conditional Transport on a Probabilistic Graph}

The ``monotone lower triangular map,'' introduced in \citet{bogachev2005triangular} could be used when dealing with time series, since there is a natural ordering between variables, indexed by the time, as discussed in \citet{backhoff2017causal} or \citet{bartl2021wasserstein}. In the general non-temporal case of time series \(X_t\), it is natural to extend that approach to acyclical probabilistic graphic models, following 
\citet{cheridito2023optimal}. Instead of two general measures \(\mu_0\) and \(\mu_1\) on \(\mathbb{R}^d\), we use only measures ``factorized according to \(\mathcal{G}\),'' some probabilistic graphical model, as defined in \citet{lauritzen2019lectures}.

\begin{definition}
    Consider some acyclical causal graph \(\mathcal{G}\) on \((s,\boldsymbol{x})\) where variables are topologically sorted, where \(s\in\{0,1\}\) is a binary variable, defining two measures \(\mu_0\) and \(\mu_1\) on \(\mathbb{R}^d\), by conditioning on \(s=0\) and \(s=1\), respectively, factorized according to \(\mathcal{G}\). Define
    \[
T^\star_{\mathcal{G}}(x_1,\cdots,x_d) = 
\begin{pmatrix}
    T^\star_{1}(x_1)\\
T^\star_{2}(x_2|~\text{\em parents}(x_{2}))\\
\vdots\\
T^\star_{{d-1}}(x_{d-1}|~\text{\em parents}(x_{d-1}))\\
T^\star_{d}(x_{d}|~\text{\em parents}(x_{d}))
\end{pmatrix}.
\]
This mapping will be called ``sequential conditional transport on the graph \(\mathcal{G}\),'' or shortly ``sequential transport.''\footnote{Given the topological order of the graph and assuming the \(\mu_0,\mu_1\)-conditionals are atomless, the existence and unicity of the sequential transport map are guaranteed, as it involves fewer conditioning variables compared to the KR transport map.}
\end{definition}

A classical algorithm for topological sorting is \citet{kahn1962topological}'s ``Depth First Search'' (DFS), and other algorithms are discussed in Section 20.4 in \citet{cormen2022introduction}. For the causal graphs of Figure~\ref{fig:DAG-2}:
\[
T^\star_{\mathcal{G}}(x_1,x_2) = 
\begin{pmatrix}
    T^\star_{1}(x_1)\\
T^\star_{2}(x_2|x_1)
\end{pmatrix},\text{ for Figure}~\ref{fig:DAG-2}\text{a},
\]
\[
T^\star_{\mathcal{G}}(x_1,x_2) = 
\begin{pmatrix}
    T^\star_{1}(x_1|x_2)\\
T^\star_{2}(x_2)
\end{pmatrix},\text{ for Figure}~\ref{fig:DAG-2}\text{b}.
\]
In that simple case, for Figure~\ref{fig:DAG-2}a, we recognize the  ``monotone lower triangular map,'' and the ``monotone upper triangular map,'' for~\ref{fig:DAG-2}b (see Section~\ref{subsec:kr-def}). Finally, for the causal graph on the German Credit dataset of Figure~\ref{fig:dag}, variables are sorted, and
\[
T^\star_{\mathcal{G}}(x_1,\cdots,x_7) = 
\begin{pmatrix}
    T^\star_{1}(x_1)\\
T^\star_{2}(x_2|x_1)\\
T^\star_{3}(x_3|x_1,x_2)\\
T^\star_{4}(x_4|x_2,x_3)\\
T^\star_{5}(x_5|x_1,x_2,x_4)\\
T^\star_{6}(x_6|x_3,x_5)\\
T^\star_{7}(x_7|x_1,x_4,x_5,x_6)\\
\end{pmatrix}.
\]
Alternatively, using the ``monotone lower triangular map'' for the German Credit dataset to compute counterfactuals suggests that the assumed DAG contains more edges than the DAG illustrated in Figure~\ref{fig:DAG-2}. In this case, the edges are specified as \( E = \{(i, j) \in V^2 : i < j\} \), with \( V = \{s, x_1, x_2, \dots, x_7\} \). Notably, multivariate OT corresponds to a fully connected graph, with \( E = \{(i, j) \in V^2 : i \neq j\} \) \cite{cheridito2023optimal}. The impact of edge mispecifications on sequential transport is examined in Appendix~\ref{appendix:wrong-assumptions}. If the graph is entirely unknown, one could infer it as discussed in \cite{Zheng_2018, Yu_2019_pmlr, cai2023on}, or use Bayes' rule to compute posterior DAGs, incorporating uncertainty quantification for counterfactuals \cite{Toth_2022_neurips}.

\begin{figure}[!h]
    \centering
   \begin{tabular}{cc}
\tikz{
    \useasboundingbox (0, -1) rectangle (3.5, 1.2);
    \node[fill=red!30] (s) at (0,0) {$S$};
    \node[fill=yellow!60] (x2) at (1.5,-.5) {$X_{2}$};
    \node[fill=yellow!60] (x1) at (1.5,.5) {$X_{1}$};
    \node[fill=blue!30] (y) at (3,0) {$Y$};
    \node[] (a) at (1.5,1) {(a)};
    \path[->, black] (s) edge (x1);
    \path[->, black] (x1) edge (x2);
    \path[->, black] (s) edge (x2);
    \path[->, black] (x1) edge (y);
    \path[->, black] (x2) edge (y);
    \path[->, black, bend right=80] (s) edge (y);
}  & \tikz{
    \useasboundingbox (0, -1) rectangle (3.5, 1.2);
    \node[fill=red!30] (s) at (0,0) {$S$};
    \node[fill=yellow!60] (x2) at (1.5,-.5) {$X_{2}$};
    \node[fill=yellow!60] (x1) at (1.5,.5) {$X_{1}$};
    \node[fill=blue!30] (y) at (3,0) {$Y$};
    \node[] (a) at (1.5,1) {(b)};
    \path[->, black] (s) edge (x1);
    \path[->, black] (x2) edge (x1);
    \path[->, black] (s) edge (x2);
    \path[->, black] (x1) edge (y);
    \path[->, black] (x2) edge (y);
    \path[->, black, bend right=80] (s) edge (y);
} 
\end{tabular}
    \caption{Two simple causal networks, with two legitimate mitigating variables, \(x_1\) and \(x_2\).}
    \label{fig:DAG-2}
\end{figure}

For the fairness application in the next section, $s$ is treated as a ``source'' with no parents, allowing it to be the first vertex in the topological ordering of the network on $(s,\boldsymbol{x})$. The counterfactual value is then derived by propagating ``downstream'' in the causal graph as $s$ changes from $0$ to $1$

\subsection{Algorithm}

\begin{algorithm}
\caption{Sequential transport on causal graph}\label{alg:1}
\begin{algorithmic}
\Require graph \(\mathcal{G}\) on \((s,\boldsymbol{x})\), with adjacency matrix \(\boldsymbol{A}\)
\Require dataset \((s_i,\boldsymbol{x}_i)\) and one individual \((s=0,\boldsymbol{a})\)
\Require bandwidths \(\boldsymbol{h}\) and \(\boldsymbol{b}_j\)'s 
\State \((s,\boldsymbol{v})\gets\boldsymbol{A}\) the topological ordering of vertices (DFS)
\State \(T_s\gets\text{identity}\)
\For{\(j\in \boldsymbol{v}\)} 
    \State \(\boldsymbol{p}(j) \gets \text{parents}(j)\)
    \State \(T_j(\boldsymbol{a}_{\boldsymbol{p}(j)})\gets (T_{\boldsymbol{p}(j)_1}(\boldsymbol{a}_{\boldsymbol{p}(j)}),\cdots,T_{\boldsymbol{p}(j)_{k_j}}(\boldsymbol{a}_{\boldsymbol{p}(j)}))\)
    \State \((x_{i,j|s},\boldsymbol{x}_{i,\boldsymbol{p}(j)|s})\gets\) subsets when \(s\in\{0,1\}\)
    \State \(w_{i,j|0}\gets \phi(\boldsymbol{x}_{i,\boldsymbol{p}(j)|0};\boldsymbol{a}_{\boldsymbol{p}(j)},\boldsymbol{b}_j)\) (Gaussian kernel) 
    \State \(w_{i,j|1}\gets \phi(\boldsymbol{x}_{i,\boldsymbol{p}(j)|1};T_j(\boldsymbol{a}_{\boldsymbol{p}(j)}),\boldsymbol{b}_j)\) 
    \State \(\hat{f}_{h_j|s}\gets \text{density estimator}\) of \(x_{\cdot,j|s}\), weights \(w_{\cdot,j|s}\). 
    \State \(\hat{F}_{h_j|s}(\cdot)\gets\displaystyle\int^{~\cdot}_{-\infty}\hat{f}_{h_j|s}(u)\mathrm{d}u\), c.d.f.
    \State \(\hat{Q}_{h_j|s}\gets \hat{F}_{h_j|s}^{-1}\), quantile
    \State \(\hat{T}_{j}(\cdot)\gets\hat{Q}_{h_j|1}\circ \hat{F}_{h_j|0}(\cdot)\)
\EndFor
\State \(\boldsymbol{a}^\star\gets (T_{1}(\boldsymbol{a}_{1}),\cdots,T_{d}(\boldsymbol{a}_{d}))\)\\
\Return \((s=1,\boldsymbol{a}^\star)\), counterfactual of \((s=0,\boldsymbol{a})\) 
\end{algorithmic}
\end{algorithm}

Algorithm~\ref{alg:1} describes this sequential approach, which can be illustrated using the DAG in Figure~\ref{fig:DAG-2}a, as shown in 
Figure~\ref{fig:gauss:algo}. The preliminary step is to determine the topological order of the causal graph. In Figure~\ref{fig:DAG-2}a the order is \((s,(x_1,x_2))\). The first step is to estimate densities \(\widehat{f}_{1|s}\) of \(x_1\) in the two groups (\(s\) being either \(0\) or \(1\)) as shown in the top left of Figure~\ref{fig:gauss:algo}. 
Next, numerical integration and inverse are used to compute the cumulative distributions \(\widehat{F}_{1|s}\) and quantile functions \(\widehat{Q}_{1|s}\). 
To compute the counterfactual for \((s=0,\boldsymbol{a})\), \({a}_1^\star\) is calculated as \(\hat{T}_{1}({a}_1)\), where \(\hat{T}_{1}(\cdot)=\hat{Q}_{1|1}\circ \hat{F}_{1|0}(\cdot)\). 
The second step involves considering the second variable in the topological order, conditional on its parents. Suppose \(x_2\) is the second variable, and for illustration that \(x_1\) is the (only) parent of \(x_2\). The densities \(\widehat{f}_{2|s}\) of \(x_2\) are then estimated in the two groups, conditional on their parents: either conditional on \(x_1={a}_1\) (subgroup \(s=0\)), or conditional on \(x_1={a}_1^\star\) (subgroup \(s=1\)). 
This is feasible since all transports of parents were computed in an earlier step. This can be visualized in the bottom left of Figure~\ref{fig:gauss:algo}. 
As in the previous step, the conditional cumulative distributions \(\widehat{F}_{2|s}\) and conditional quantile functions \(\widehat{Q}_{2|s}\) (conditional on the parents) are computed. Then \({a}_2^\star\) is determined as \(\hat{T}_{2}({a}_2)\) where \(\hat{T}_{2}(\cdot)=\hat{Q}_{2|1}\circ \hat{F}_{2|0}(\cdot)\). This process is repeated until all variables have been considered. At the end, starting from an individual with features \(\boldsymbol{x}=\boldsymbol{a}\), in group \(s=0\), the counterfactual version in group \(s=1\) is obtained, with transported features, {\em mutatis mutandis}, \(\boldsymbol{a}^\star\). As the number of parents per variable in the DAG increases, calculating conditional distributions for a variable becomes complex and less robust. Handling categorical variables in counterfactuals is detailed in Appendix~\ref{sec:appendix:algo}, enabling the application of sequential transport to datasets like \texttt{adult income} and \texttt{COMPAS} in Appendix~\ref{appendix:add-real-data}.

\begin{figure}
    \centering
    \includegraphics[width=.3\linewidth]{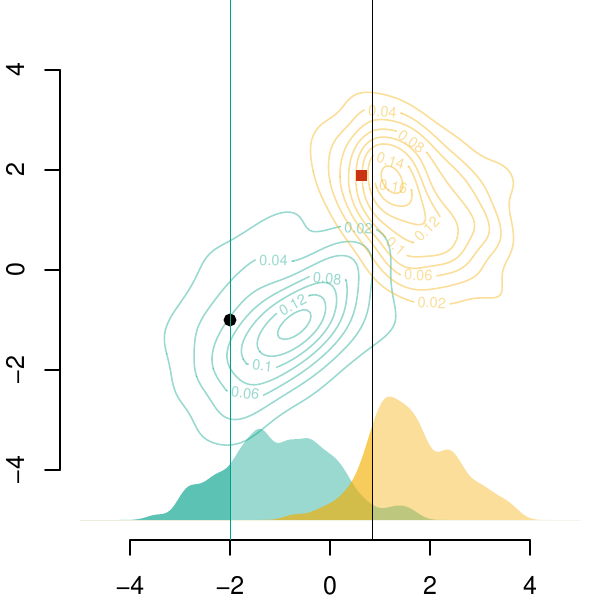}%
    \includegraphics[width=.3\linewidth]{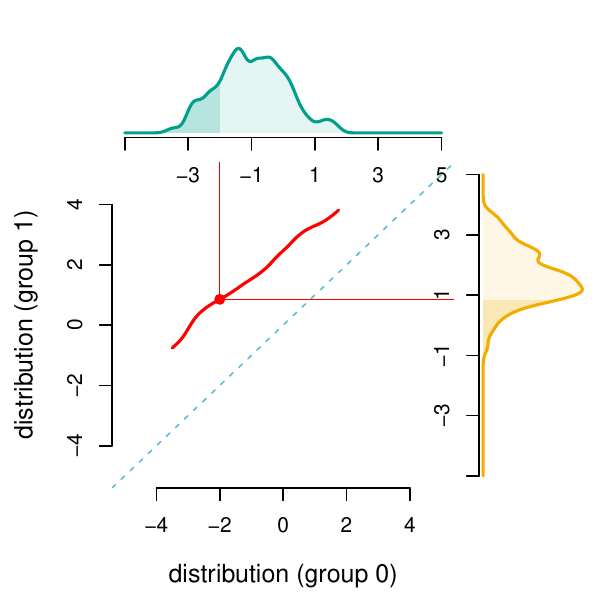}
    \includegraphics[width=.3\linewidth]{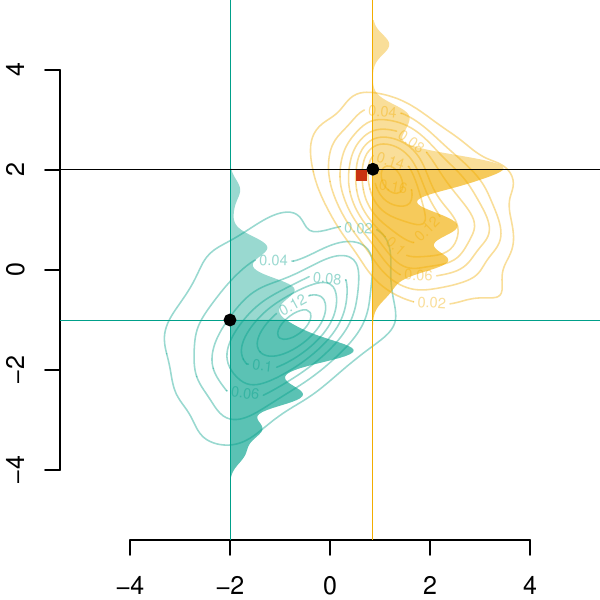}%
        \includegraphics[width=.3\linewidth]{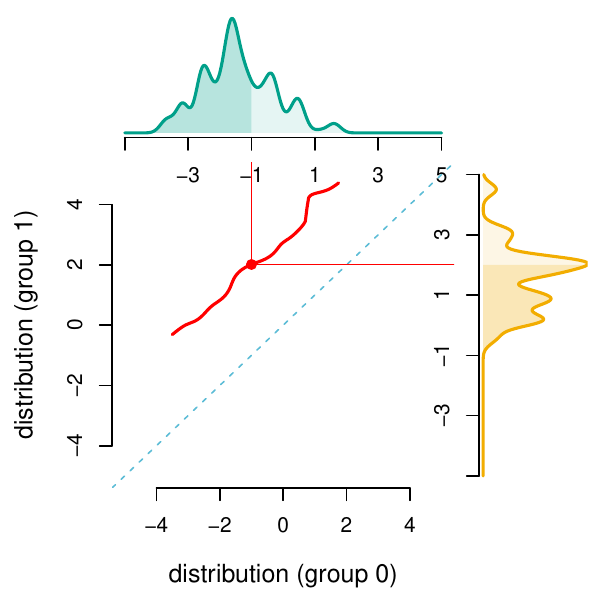}
    \caption{Illustration of Algorithm~\ref{alg:1} for DAG in  Figure~\ref{fig:DAG-2}a, with simulated data; first step at the top, second step at the bottom. The red square represents the multivariate OT of the bottom-left point.}
    \label{fig:gauss:algo}
\end{figure}

\section{Interpretable Counterfactual Fairness}\label{sec:fairness}

\subsection{Individual Counterfactual Fairness}

\subsubsection{General Context}

Following \citet{dwork2012fairness}, a fair decision means that ``similar individuals'' are treated similarly. As discussed in the introduction, \citet{Kusner17} and \citet{russell2017worlds} considered a ``counterfactual fairness'' criterion. Based on the approach discussed above, it is possible to quantify unfairness, for a single individual, of a model \(m\), trained on features \((s,\boldsymbol{x})\) to predict an outcome \(y\). If \(y\in\{0,1\}\) is binary, then \(m\) represents the underlying score, corresponding to the conditional probability that \(y=1\).

\subsubsection{Illustration With Simulated Data}

Consider the causal graphs in  Figure~\ref{fig:DAG-2}, with one sensitive attribute \(s\), two legitimate features \(x_1\) and \(x_2\) and one outcome \(y\). Here, \(y\) is the score obtained from a logistic regression, specifically,
\[
m(x_1,x_2,s)=\big(1+\exp\big[-\big((x_1+x_2)/2 + \boldsymbol{1}(s=1)\big)\big]\big)^{-1}.
\]

Iso-scores can be visualized at the top of Figure~\ref{fig:gauss:logistic:2}, with group \(0\) on the left, \(1\) on the right. Consider an individual \((s,x_1,x_2)=(s=0,-2,-1)\) in group \(0\), with a score of \(18.24\%\) (bottom left of Figure~\ref{fig:gauss:logistic:2}). Using Algorithm~\ref{alg:1}, its counterfactual counterpart \((s=1,x^\star_1,x^\star_2)\) can be constructed. The resulting score varies depending on the causal assumption. The score would be \(61.40\%\) assuming the causal graph of Figure~\ref{fig:DAG-2}a, and \(56.34\%\) assuming causal graph~\ref{fig:DAG-2}b. 
In the first case, the {\em mutatis mutandis} difference \(m(s=1,x^\star_1,x^\star_2)-m(s=0,x_1,x_2)\), i.e., \(+43.15\%\), is:
\begin{eqnarray*}
    && m(s=1,x_1,x_2) - m(s=0,x_1,x_2) ~~~ :-10.66\%\\ 
    &+& m(s=1,x^\star_1,x_2) - m(s=1,x_1,x_2)~~~ :+15.63\%\\
    &+& m(s=1,x^\star_1,x^\star_2) - m(s=1,x^\star_1,x_2)~~~ :+38.18\%.
\end{eqnarray*}
The first term is the {\em ceteris paribus} difference, the second one the change in \(x_1\) and the third one the change in \(x_2\), conditional on the change in \(x_1\). 
If, instead, we assume the causal graph of Figure~\ref{fig:DAG-2}b, the score of the same individual would become \(56.34\%\) and the {\em mutatis mutandis} difference \(m(s=1,x^\star_1,x^\star_2)-m(s=0,x_1,x_2)\), i.e., \(+38.09\%\), is:
\begin{eqnarray*}
    && m(s=1,x_1,x_2) - m(s=0,x_1,x_2) ~~~ :-10.66\%\\ 
    &+& m(s=1,x_1,x_2^\star) - m(s=1,x_1,x_2)~~~ :+14.51\%\\
    &+& m(s=1,x^\star_1,x^\star_2) - m(s=1,x_1,x_2^\star)~~~ :+34.24\%.
\end{eqnarray*}
At the bottom right of Figure~\ref{fig:gauss:logistic:2}, the {\em mutatis mutandis} impact on the scores can be visualized.

\begin{figure}[htb]
    \centering
    \includegraphics[width=.5\textwidth]{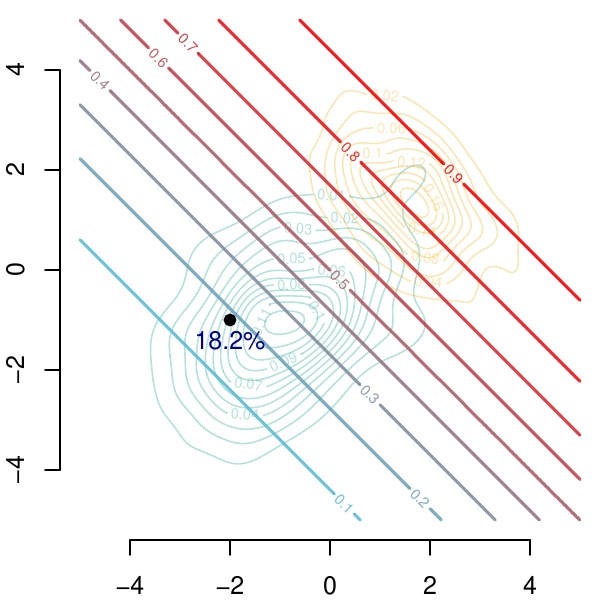}\includegraphics[width=.5\textwidth]{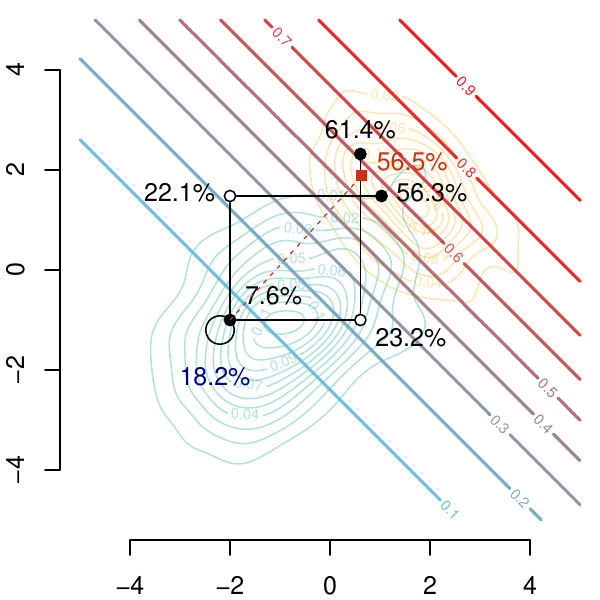}
    \caption{In the background, level curves for \((x_1,x_2)\mapsto m(0,x_1,x_2)\) and \(m(1,x_1,x_2)\) respectively on the left and on the right. Then, on the left, individual \((s,x_1,x_2)=(s=0,-2,-1)\) (predicted 18.2\% by model \(m\)), and on the right, visualization of two counterfactuals \((s=1,x_1^\star,x_2^\star)\) according to causal graphs~\ref{fig:DAG-2}a (bottom right path, predicted 61.4\%) and~\ref{fig:DAG-2}b (top left path, predicted 56.3\%). The red dot is the counterfactual obtained with multivariate OT (predicted 56.5\%).}
    \label{fig:gauss:logistic:2}
\end{figure}

\subsection{Global Fairness Metrics}
\label{subsec:fairnessmetrics}

Instead of focusing on a single individual, it is possible to quantify the fairness of a model \(m\) on a global scale. For example, the Demographic Parity criterion can be extended to Counterfactual Demographic Parity (CDP), allowing fairness assessment within a population subgroup with $s=0$.
Consider the empirical version of ``counterfactual fairness'' in \citet{Kusner17}
\begin{align} \label{equ:dp}
    \mathrm{CDP}=\displaystyle{\frac{1}{n_0}\sum_{i\in\mathcal{D}_0}m(1,\boldsymbol{x}_{i}^\star) - m(0,\boldsymbol{x}_{i}) },
\end{align}
which corresponds to the ``average treatment effect of the treated'' in the classical causal literature. This can be computed more efficiently using {Algorithm~2 in Appendix~\ref{sec:appendix:algo}, which offers a faster alternative compared to Algorithm~1. Other group fairness metrics, based on Equalized Odds, can be extended to aggregated counterfactual fairness measures (see Appendix~\ref{appendix:fairness-metrics}).

\section{Application on Real Data}\label{sec:real-data}

We analyze the Law School Admission Council dataset \citep{Wightman1998LSACNL}, focusing on four variables: race \(s \in \left\{\text{Black},\text{White}\right\}\) (corresponding to 0 and 1), undergraduate GPA before law school (\(x_1\), UGPA), Law School Admission Test (\(x_2\), LSAT), and a binary response (\(y\)) indicating whether the first-year law school grade (FYA) is above the median, as described in \citet{black2020fliptest}.  Unlike \citet{de2024transport, black2020fliptest, Kusner17}, we assume the causal graph in Figure~\ref{fig:DAG-LAW}, where UGPA influences LSAT. We aim to evaluate counterfactual fairness for Black individuals in logistic regression predictions (\(\hat{y}|s = 0\)), comparing an ``aware'' classifier, i.e., that includes \(s\) among the explanatory variables, with an ``unaware'' model that considers only \(\boldsymbol{x}=(x_1,x_2)\). Fairness is measured using \(\mathrm{CDP}\) (see Eq.~\ref{equ:dp}). We apply the sequential transport method from Algorithm~2 to compute counterfactuals \(\hat{y}^\star(s = 1)|s = 0\) following the network's topological order in Figure~\ref{fig:DAG-LAW}. These results are compared with those obtained from multivariate OT \citep{de2024transport} and quantile regressions \citep{plevcko2021fairadapt}, namely Fairadapt.

\begin{figure}[htb]
    \centering
    \tikz{
    \useasboundingbox (0, -1) rectangle (4, .7);
    \node[fill=red!30] (s) at (0,0) {race $S$};
    \node[fill=yellow!60] (x2) at (2,-.5) {LSAT $X_{2}$};
    \node[fill=yellow!60] (x1) at (2,.5) {UGPA  $X_{1}$};
    \node[fill=blue!30] (y) at (4,0) {FYA $Y$ };
    \path[->, black] (s) edge (x1);
    \path[->, black] (x1) edge (x2);
    \path[->, black] (s) edge (x2);
    \path[->, black] (x1) edge (y);
    \path[->, black] (x2) edge (y);
    \path[->, black, bend right=50] (s) edge (y);
}  
    \caption{Causal graph of the Law School dataset.}\label{fig:DAG-LAW}
\end{figure}

\begin{figure}[htb]
    \centering
    \includegraphics[width=.5\linewidth]{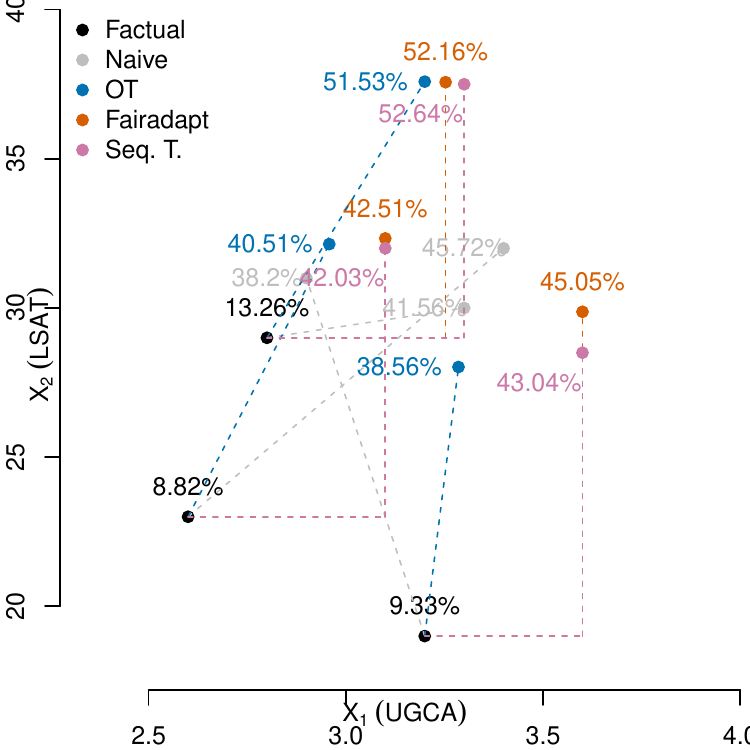}\includegraphics[width=.5\linewidth]{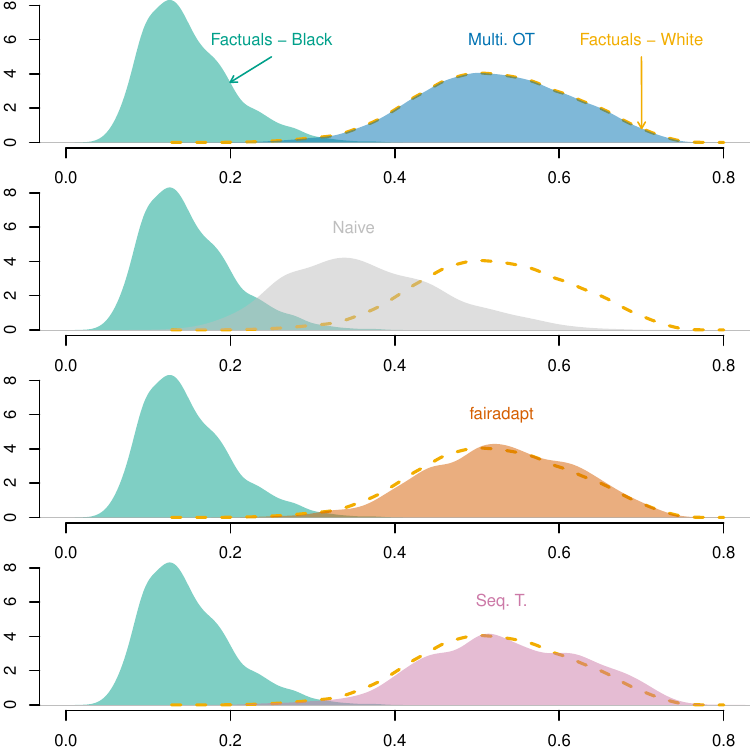}
    \caption{Counterfactual calculations for three Black individuals on the left (percentages indicate predicted scores), and densities of predicted scores (aware model) for all Black individuals with factuals and counterfactuals on the right. The dashed line represents the density of predicted scores for the observed White individuals.}
    \label{fig:densities-law-pred}
\end{figure}

\begin{table}[htb]
    \centering
    \begin{tabular}{cccc}
        \hline
        & Fairadapt & multi. OT & seq. T \\ \hline
       Aware model & 0.3810 & 0.3727 & 0.3723\\
       Unaware model & 0.1918 & 0.1821 & 0.1817 \\ \hline 
    \end{tabular}
    \caption{\(\mathrm{CDP}\) for Black individuals from Eq.~\ref{equ:dp} 
    comparing classifier predictions over original features \(\boldsymbol{x}\) (resp. \((s=0,\boldsymbol{x})\)) and their counterfactuals \(\boldsymbol{x}^\star\) (resp. \((s=1,\boldsymbol{x}^\star)\)), using Fairadapt, multivariate OT, and sequential transport.}
    \label{tab:dpcounterfactual}
\end{table}

Figure~\ref{fig:densities-law-pred} illustrates the similarity between Fairadapt and sequential transport, both assuming a DAG, as shown by the counterfactual pathways for three Black individuals (left) and the alignment of counterfactual predicted score densities (right). The density of multivariate OT counterfactuals resembles factual White outcome distribution due to its matching process. Overall, the three methods yield similar results, as reflected in the aggregated counterfactual fairness metric in Table~\ref{tab:dpcounterfactual}. Lastly, the ``aware'' model, which directly incorporates \(s\) into its covariates, is less counterfactually fair than the ``unaware'' model.

\section*{Conclusion}

In this paper, we propose a sequential transport approach for constructing counterfactuals based on OT theory while respecting the underlying causal graph of the data. By using conditional univariate transport maps, we derive closed-form solutions for each coordinate of an individual's characteristics, which facilitates the interpretation of both individual counterfactual fairness of our predictive model, and global fairness through ``counterfactual demographic parity.'' Future work could extend counterfactual fairness evaluation to mitigation by applying pre-processing or in-processing methods using sequential transport for counterfactual generation.

\appendix

\section{Gaussian Case}\label{sec:appendix-gaussian}

The Gaussian case is the most simple one since mapping \({T}^\star\), corresponding to OT, can be expressed analytically (it will be a linear mapping). Furthermore, conditional distributions of a multivariate Gaussian distribution are Gaussian distributions, and that can be used to consider an iteration of simple conditional (univariate) transports, as a substitute to joint transport \({T}^\star\). Here \(\Phi\) denotes the univariate cumulative distribution function of the standard Gaussian distribution \(\mathcal{N}(0,1)\).

\subsection{Univariate Optimal Gaussian Transport}

One can easily prove that the optimal mapping, from a \(\mathcal{N}(\mu_0,\sigma_0^2)\) to a \(\mathcal{N}(\mu_1,\sigma_1^2)\) distribution is (see Figure~\ref{fig:transport}):
\[
x_{1}={T}^\star(x_{0})= \mu_{1}+\frac{\sigma_{1}}{\sigma_{0}}(x_{0}-\mu_{0}),
\]
which is a nondecreasing linear transformation.

\subsection{Multivariate Optimal Gaussian Transport}

Recall that \(\boldsymbol{X}\sim\mathcal{N}(\boldsymbol{\mu},\boldsymbol{\Sigma})\), \(\boldsymbol{B} = \boldsymbol{\Sigma}^{-1}\), if its density, with respect to Lebesgue measure is
\begin{equation}\label{eq:gaussian}
{\displaystyle f(\boldsymbol{x})\propto{{\exp \left(-{\frac {1}{2}}\left({\boldsymbol{x} }-{\boldsymbol {\mu }}\right)^{\top}{\boldsymbol {B}}\left({\boldsymbol {x} }-{\boldsymbol {\mu }}\right)\right)}.%
}
}
\end{equation}

If \(\boldsymbol{X}_0\sim\mathcal{N}(\boldsymbol{\mu}_0,\boldsymbol{\Sigma}_0)\) and \(\boldsymbol{X}_1\sim\mathcal{N}(\boldsymbol{\mu}_1,\boldsymbol{\Sigma}_1)\), the optimal mapping is also linear,
\[
\boldsymbol{x}_{1} = T^\star(\boldsymbol{x}_{0})=\boldsymbol{\mu}_{1} + \boldsymbol{A}(\boldsymbol{x}_{0}-\boldsymbol{\mu}_{0}),
\]
where \(\boldsymbol{A}\) is a symmetric positive matrix that satisfies \(\boldsymbol{A}\boldsymbol{\Sigma}_{0}\boldsymbol{A}=\boldsymbol{\Sigma}_{1}\), which has a unique solution given by \(\boldsymbol{A}=\boldsymbol{\Sigma}_{0}^{-1/2}\big(\boldsymbol{\Sigma}_{0}^{1/2}\boldsymbol{\Sigma}_{1}\boldsymbol{\Sigma}_{0}^{1/2}\big)^{1/2}\boldsymbol{\Sigma}_{0}^{-1/2}\), where \(\boldsymbol{M}^{1/2}\) is the square root of the square (symmetric) positive matrix \(\boldsymbol{M}\) based on the Schur decomposition (\(\boldsymbol{M}^{1/2}\) is a positive symmetric matrix), as described in \citet{higham2008functions}. If \(\boldsymbol{\Sigma}=\displaystyle\begin{pmatrix}1&r\\r&1\end{pmatrix}\), and if \(a=\sqrt{(1-\sqrt{1-r^2})/2}\), then:
\[
\boldsymbol{\Sigma}^{1/2}=\displaystyle\begin{pmatrix}\sqrt{1-a^2}&a\\a&\sqrt{1-a^2}\end{pmatrix}.
\]

Observe further this mapping is the gradient of the convex function
\[
\psi(\boldsymbol{x})=\frac{1}{2}(\boldsymbol{x}-\boldsymbol{\mu}_0)^\top\boldsymbol{A}(\boldsymbol{x}-\boldsymbol{\mu}_0)+\boldsymbol{x}-\boldsymbol{\mu}_1^\top\boldsymbol{x}
\]
and \(\nabla T^\star = \boldsymbol{A}\) (see \citet{takatsu2011wasserstein} for more properties of Gaussian transport).
And if \(\boldsymbol{\mu}_0=\boldsymbol{\mu}_1=\boldsymbol{0}\), and if \(\boldsymbol{\Sigma}_{0}=\mathbb{I}\) and \(\boldsymbol{\Sigma}_{1}=\boldsymbol{\Sigma}\), \(\boldsymbol{x}_{1} = T^\star(\boldsymbol{x}_{0})=\boldsymbol{\Sigma}^{1/2}\boldsymbol{x}_{0}\). Hence, \(\boldsymbol{\Sigma}^{1/2}\) is a linear operator that maps from \(\boldsymbol{X}_0\sim\mathcal{N}(\boldsymbol{0},\mathbb{I})\) (the reference density) to \(\boldsymbol{X}_1\sim\mathcal{N}(\boldsymbol{0},\boldsymbol{\Sigma})\) (the target density).

\subsection{Conditional Gaussian Transport}\label{sec:cond:gauss}

Alternatively, since \(\boldsymbol{\Sigma}\) is a positive definite matrix, from the Cholesky decomposition, it can be written as the product of a lower (or upper) triangular matrix and its conjugate transpose,
\[
\boldsymbol{\Sigma}=\boldsymbol{L}\boldsymbol{L}^\top=\boldsymbol{U}^\top\boldsymbol{U}.
\]

\begin{remark}
If \(\boldsymbol{\Sigma}=\displaystyle\begin{pmatrix}1&r\\r&1\end{pmatrix}\), then \(\boldsymbol{L}=\boldsymbol{\Sigma}_{2|1}^{1/2}=\displaystyle\begin{pmatrix}1&0\\r&\sqrt{1-r^2}\end{pmatrix}\) while \(\boldsymbol{U}=\boldsymbol{\Sigma}_{1|2}^{1/2}=\boldsymbol{\Sigma}_{2|1}^{1/2\top}=\boldsymbol{L}^\top\). Then \(\boldsymbol{L}\boldsymbol{L}^\top=\boldsymbol{\Sigma}=\boldsymbol{U}^\top\boldsymbol{U}\).
\end{remark}

Both \(\boldsymbol{L}\) and \(\boldsymbol{U}\) are linear operators that map from \(\boldsymbol{X}_0\sim\mathcal{N}(\boldsymbol{0},\mathbb{I})\) (the reference density) to \(\boldsymbol{X}_1\sim\mathcal{N}(\boldsymbol{0},\boldsymbol{\Sigma})\) (the target density). \(\boldsymbol{x}_0\mapsto\boldsymbol{L}\boldsymbol{x}_0\) and \(\boldsymbol{x}_0\mapsto\boldsymbol{U}\boldsymbol{x}_0\) are respectively linear lower and upper triangular transport maps.

More generally,  in dimension 2, consider the following (lower triangular) mapping \(T(x_0,y_0) = (T_x(x_0),T_{y|x}(y_0|x_0))\),  
\begin{eqnarray*}
&&
\mathcal{N}\left(
\begin{pmatrix}
    \mu_{0x}\\
    \mu_{0y}\\
\end{pmatrix},
\begin{pmatrix}
    \sigma_{0x}^2 & r_0\sigma_{0x}\sigma_{0y}\\
    r_0\sigma_{0x}\sigma_{0y} & \sigma_{0y}^2\\
\end{pmatrix}
\right)
\\
&&\overset{T}{\longrightarrow}
\mathcal{N}\left(
\begin{pmatrix}
    \mu_{1x}\\
    \mu_{1y}\\
\end{pmatrix},
\begin{pmatrix}
    \sigma_{1x}^2 & r_1\sigma_{1x}\sigma_{1y}\\
    r_1\sigma_{1x}\sigma_{1y} & \sigma_{1y}^2\\
\end{pmatrix}
\right),
\end{eqnarray*}

where
\[
\begin{cases}
    T_x(x_0) = \mu_{1x} +\displaystyle\frac{\sigma_{1x}}{\sigma_{0x}}(x_0-\mu_{0x})\\
    T_{y|x}(y_0) = \mu_{1y}+\displaystyle\frac{r_1\sigma_{1y}}{\sigma_{1x}}(T_x(x_0)-\mu_{1x})\phantom{\displaystyle\int}\\
    +\sqrt{\displaystyle\frac{\sigma_{0x}^2(\sigma_{1y}^2{\sigma_{1x}^2}-{r_1^2\sigma_{1y}^2})}{(\sigma_{0y}^2{\sigma_{0x}^2}-{r_0^2\sigma_{0y}^2})\sigma_{1x}^2}}(y_0\!-\!\mu_{0y}-\displaystyle\frac{r_0\sigma_{0y}}{\sigma_{0x}}(x_0\!-\!\mu_{0x}))
\end{cases}
\]
that are both linear mappings. This can be visualized on the left side of Figure~\ref{fig:gauss:seq}.

\begin{figure}
    \includegraphics[width=.5\columnwidth]{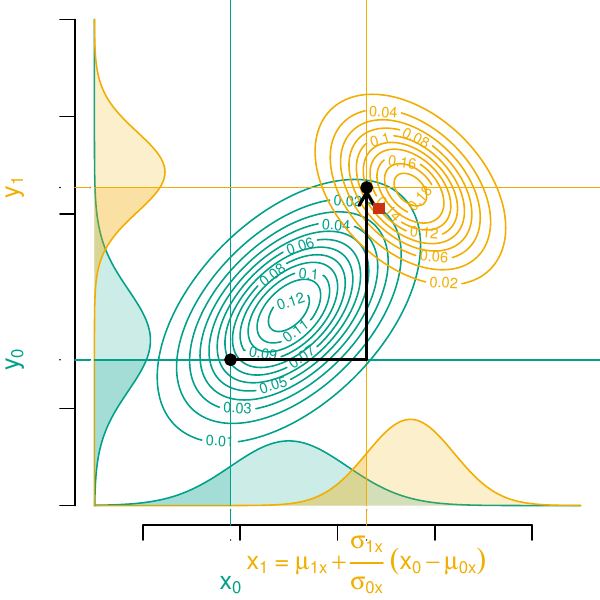}\includegraphics[width=.5\columnwidth]{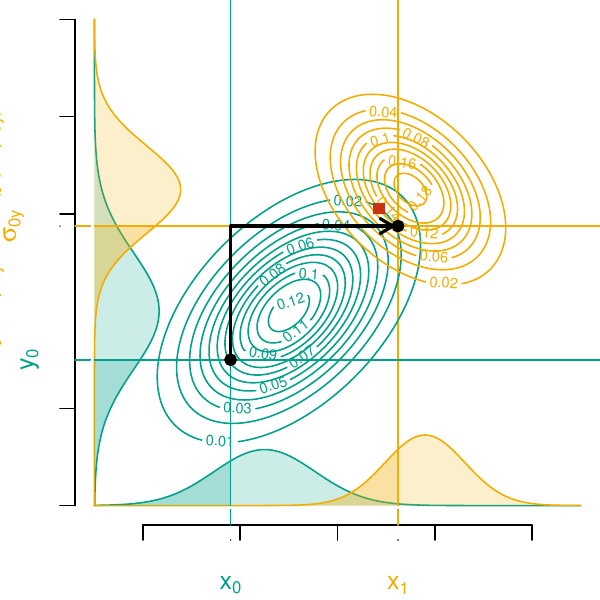}
    \caption{Two Gaussian conditional OTs. On the left-hand side, the process begins with a univariate transport along the \(x\) axis (using \(T^\star_x\)), followed by a transport along the \(y\) axis on the conditional distributions (using \(T_{y|x}^\star\)), corresponding to the ``lower triangular affine mapping.'' On the right-hand side, the sequence is reversed: it starts with a univariate transport along the \(y\) axis (using \(T^\star_y\)) followed by transport along the \(x\) axis on the conditional distributions (using \(T_{x|y}^\star\)). The red square is the multivariate OT of the point in the bottom left, corresponding to the ``upper triangular affine mapping.''}
    \label{fig:gauss:seq}
\end{figure}

Of course, this is highly dependent on the axis parametrization. Instead of considering projections on the axis, one could consider transport in the direction \(\Vec{u}\), followed by transport in the direction \(\Vec{u}^\perp\) (on conditional distributions). This can be visualized in Figure~\ref{fig:gauss:seq:rotation}.

\subsection{Gaussian Probabilistic Graphical Models}

An interesting feature of the Gaussian multivariate distribution is that any marginal and any conditional distribution (given other components) is still Gaussian. More precisely, if 
\[
{\displaystyle \boldsymbol {x} ={\begin{pmatrix}\boldsymbol {x} _{1}\\\boldsymbol {x} _{2}\end{pmatrix}}},~{\displaystyle {\boldsymbol {\mu }}={\begin{pmatrix}{\boldsymbol {\mu }}_{1}\\{\boldsymbol {\mu }}_{2}\end{pmatrix}}}\text{ and } {\displaystyle {\boldsymbol {\Sigma }}={\begin{pmatrix}{\boldsymbol {\Sigma }}_{11}&{\boldsymbol {\Sigma }}_{12}\\{\boldsymbol {\Sigma }}_{21}&{\boldsymbol {\Sigma }}_{22}\end{pmatrix}}},
\]
then \(\boldsymbol{X}_1\sim\mathcal{N}(\boldsymbol{\mu}_1,\boldsymbol{\Sigma}_{11})\), while, with notations of Eq.~\ref{eq:gaussian}, we can also write \({\displaystyle { {\boldsymbol {B }_{1}}}={\boldsymbol {B }}_{11}-{\boldsymbol {B }}_{12}{\boldsymbol {B }}_{22}^{-1}{\boldsymbol {B }}_{21}}\) (based on properties of inverses of block matrices, also called the Schur complement of a block matrix). Furthermore, conditional distributions are also Gaussian, \(\boldsymbol{X}_1|\boldsymbol{X}_2=\boldsymbol{x}_2\sim\mathcal{N}({\boldsymbol {\mu }_{1|2}},{\boldsymbol {\Sigma }_{1|2}}),\)
\[
\begin{cases}
   {\displaystyle { {\boldsymbol {\mu }_{1|2}}}={\boldsymbol {\mu }}_{1}+{\boldsymbol {\Sigma }}_{12}{\boldsymbol {\Sigma }}_{22}^{-1}\left(\boldsymbol{x}_2 -{\boldsymbol {\mu }}_{2}\right)}
   \\
   {\displaystyle { {\boldsymbol {\Sigma }_{1|2}}}={\boldsymbol {\Sigma }}_{11}-{\boldsymbol {\Sigma }}_{12}{\boldsymbol {\Sigma }}_{22}^{-1}{\boldsymbol {\Sigma }}_{21},}
\end{cases}
\]
and the inverse of the conditional variance is simply \(\boldsymbol {B }_{11}\).

It is well known that if \(\boldsymbol{X}\sim\mathcal{N}(\boldsymbol{\mu},\boldsymbol{\Sigma})\), \(X_i\indep X_j\) if and only if \(\Sigma_{i,j}=0\). More interestingly, we also have the following result, initiated by \citet{dempster1972covariance}:

\begin{proposition}
If \(\boldsymbol{X}\sim\mathcal{N}(\boldsymbol{\mu},\boldsymbol{\Sigma})\), with notations of Eq.~\ref{eq:gaussian}, \(\boldsymbol{B}=\boldsymbol{\Sigma}^{-1}\), \(\boldsymbol{X}\) is Markov with respect to \(\mathcal{G}=(E,V)\) if and only if \(B_{i,j}=0\) whenever \((i,j),(j,i)\notin E\).
\end{proposition}

\begin{proof}
This is a direct consequence of the following property : if \(\boldsymbol{X}\sim\mathcal{N}(\boldsymbol{\mu},\boldsymbol{\Sigma})\), \(X_i\indep X_j |\boldsymbol{X}_{-i,j}\) if and only if \(B_{i,j}=0\) (since the log-density has separate terms in \(x_i\) and \(x_j\)).   
\end{proof}

\subsection{Sequential Transport}

In the Gaussian case we obviously recover the results of Section~\ref{sec:cond:gauss}, if we plug Gaussian distributions in the expressions of Section~\ref{sec:seq:transport}
\[
\begin{aligned}
   X_{0:1}\!\sim\!\mathcal{N}(\mu_{0:1},\sigma_{0:1}^2),\text{ hence }F_{0:1}(x)\!=\!\Phi\big(\sigma_{0:1}^{-1}(x\!-\!\mu_{0:1})\big)\\
   X_{1:1}\!\sim\!\mathcal{N}(\mu_{1:1},\sigma_{1:1}^2),\text{ hence }F_{1:1}^{-1}(u)\!=\!\mu_{1:1}\!+\!\sigma_{1:1}\!\Phi^{-1}(u)
\end{aligned}
\]
thus
\[
T_1^\star(x) = F_{1:1}^{-1}\big(F_{0:1}(x)\big)=\mu_{1:1} +\displaystyle\frac{\sigma_{1:1}}{\sigma_{0:1}}(x-\mu_{0:1}),
\]
while
\[
\begin{cases}
   X_{0:2}|x_{0:1}\sim\mathcal{N}(\mu_{0:2|1},\sigma_{0:2|1}^2),\\
   X_{1:2}|x_{0:1}\sim\mathcal{N}(\mu_{1:2|1},\sigma_{1:2|1}^2),\\ 
\end{cases}
\]
i.e.,
\[
\begin{cases}
  F_{0:2|1}(x)=\Phi\big(\sigma_{0:2|1}^{-1}(x-\mu_{0:2|1})\big),\\
  F_{1:2|1}^{-1}(u)=\mu_{1:2|1}+\sigma_{1:2|1}\Phi^{-1}(u),
\end{cases}
\]
where we consider \(X_{0:2}\) conditional to \(X_{0:1}=x_{0:1}\) in the first place, 
\[
\begin{cases}
   \mu_{0:2|1} = \mu_{0:2} +\displaystyle\frac{\sigma_{0:2}}{\sigma_{0:1}}(x_{0:1}-\mu_{0:1}),\phantom{\int}\\
   \sigma_{0:2|1}^2 = \sigma_{0:2}^2 -\displaystyle\frac{r^2_0\sigma_{0:2}^2}{\sigma_{0:1}^2},\phantom{\int}
\end{cases}
\]
and \(X_{1:2}\) conditional to \(X_{1:1}=T^\star_1(x_{0:1})\) in the second place, 
\[
\begin{cases}
   \mu_{1:2|1} = \mu_{1:2} +\displaystyle\frac{\sigma_{1:2}}{\sigma_{1:1}}\big(T^\star_1(x_{0:1})-\mu_{1:1}\big),\phantom{\int}\\
   \sigma_{1:2|1}^2= \sigma_{1:2}^2 -\displaystyle\frac{r^2_1\sigma_{1:2}^2}{\sigma_{1:1}^2},\phantom{\int}
\end{cases}
\]
thus
\[
T_{2|1}(x) = F_{1:2|1}^{-1}\big(F_{0:2|1}(x)\big)=\mu_{1:2|1} +\displaystyle\frac{\sigma_{1:2|1}}{\sigma_{0:2|1}}(x-\mu_{0:2|1}),
\]
which is
\begin{eqnarray*}
&&\mu_{1:2}+\displaystyle\frac{r_1\sigma_{1:2}}{\sigma_{1:1}}\big(\mu_{1:1} +\displaystyle\frac{\sigma_{1:1}}{\sigma_{0:1}}(x_{0:1}-\mu_{0:1})-\mu_{1:1}\big) \\&&+\sqrt{\frac{\sigma_{0:1}^2(\sigma_{1:2}^2{\sigma_{1:1}^2}-{r_1^2\sigma_{1:2}^2})}{(\sigma_{0:2}^2{\sigma_{0:1}^2}-{r_0^2\sigma_{0:2}^2})\sigma_{1:1}^2}}\\&&\times\big(x-\mu_{0:2}-\displaystyle\frac{r_0\sigma_{0:2}}{\sigma_{0:1}}(x_{0:1}-\mu_{0:1})\big).    
\end{eqnarray*}

\subsection{General Conditional Transport}

An interesting property of Gaussian vectors is the stability under rotations. In dimension 2, instead of a sequential transport of \(\boldsymbol{x}\) on \(\Vec{e}_x\) and then (conditionally) on  \(\Vec{e}_y\), one could consider a projection on any unit vector \(\Vec{u}\) (with angle \(\theta\)), and then (conditionally) along the orthogonal direction \(\displaystyle \Vec{u}^\perp\).
In Figure~\ref{fig:gauss:seq:rotation}, we can visualize the set of all counterfactuals \(\boldsymbol{x}^\star\) when \(\theta\in[0,2\pi]\). The (global) OT is also considered.

\begin{figure}[htb]
    \centering
    \includegraphics[width=.5\columnwidth]{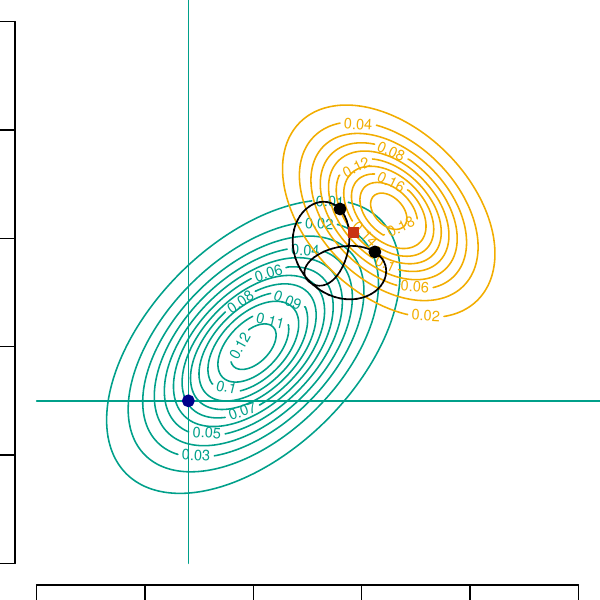}\includegraphics[width=.5\columnwidth]{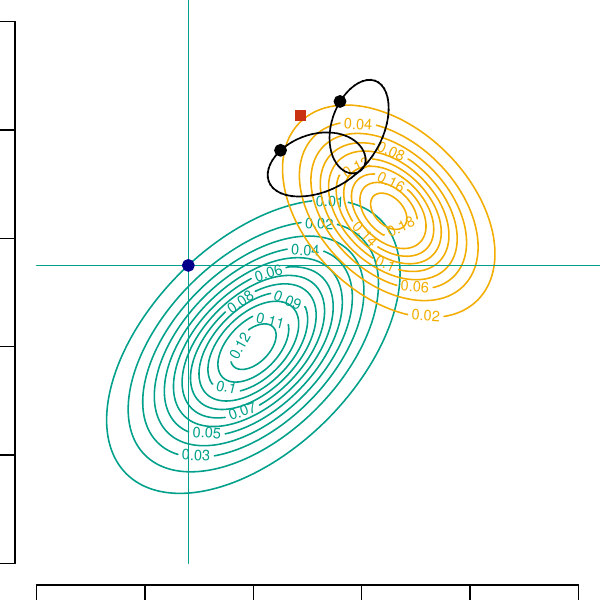}
    \caption{Gaussian conditional OTs. Each graph illustrates the transport starting from a different point (black point in the bottom left corner). The process begins with a univariate transport along the direction \(\Vec{u}\) (using \( \displaystyle T^\star_{\Vec{u}}\)) followed by a transport along the orthogonal direction \(\displaystyle \Vec{u}^\perp\), on conditional distributions (using \(\displaystyle T_{\Vec{u}^\perp|{\Vec{u}}}\)). The curves in the upper right corner of each panel represent the set of all transport maps of the same point (bottom left corner) for all possible directions \(\displaystyle \Vec{u}\), the black points correspond to classical \(x\) (horizontal) and \(y\) (vertical) directions. The red point corresponds to the global OT.}
    \label{fig:gauss:seq:rotation}
\end{figure}

\section{Algorithms}\label{sec:appendix:algo}

While Algorithm~\ref{alg:1} is intuitive, it becomes inefficient when computing counterfactuals for thousands of individuals, as conditional densities, c.d.f.'s and quantile functions should be computed for each individual. An alternative is to compute these quantities on a grid, store them, and then retrieve them as needed. 
These objects are generated using Algorithm~\ref{alg:seqt-grid} (see also Figures~\ref{fig:tensor:x1} and~\ref{fig:tensor:x2}, which visualize the vectors \(F_{1|0}\), \(F_{1|1}\), \(F_{2|0}[\cdot,i]\) and \(F_{2|1}[\cdot,j]\)---with light colors corresponding to small probabilities and darker ones to large probabilities). Algorithm~\ref{alg:3} computes the counterfactual for a given individual with \(s=0\) using the stored functions. In Algorithm~\ref{alg:seqt-grid} if $j$ has no parents, \(\mathrm{parents}(j)=\varnothing\), \(d_j=0\) and then \(F_{j|s}\) and \(Q_{j|s}\) are vectors (of length \(k\)). In Algorithm~\ref{alg:3}, in that case, \(\boldsymbol{i}_0=\boldsymbol{i}_1=\varnothing\).

\begin{algorithm}
\caption{Faster sequential transport %on causal graph
on grids (1)}\label{alg:seqt-grid}
\begin{algorithmic}
\Require graph on \((s,\boldsymbol{x})\), with adjacency matrix \(\boldsymbol{A}\)
\Require dataset \((s_i,\boldsymbol{x}_i)\) and \(k\in\mathbb{N}\) some grid size,
\Require grids \(\boldsymbol{g}_{j|s}=({g}_{j,1|s},\cdots,{g}_{j,k|s})\), for all variable \(j\) \Require grid \(\boldsymbol{u}=(1,\cdots,k)/(k+1)\), for all \(j\) 
\State \((s,\boldsymbol{v})\gets\boldsymbol{A}\) the topological ordering of vertices (DFS)
\For{\(j\in \boldsymbol{v}\)} 
    \State \(\boldsymbol{p}(j) \gets \text{parents}(j)\), dimension \(d_j\)
    \State \(\mathcal{G}_{j|s}\gets\) grid \(\boldsymbol{g}_{\boldsymbol{p}(j)_1|s}\times\cdots\times\boldsymbol{g}_{\boldsymbol{p}(j)_{d_j}|s}\)
    \State \(F_{j|s}\gets\) tensors \(k\times k^{d_j}\), taking values in \(\boldsymbol{u}\)
    \State \(Q_{j|s}\gets\) tensors \(k\times k^{d_j}\), taking values in \(\boldsymbol{g}_{j|s}\)
    \For{\(\boldsymbol{i}=(i_1,\cdots,i_{d_j})\in\{1,\cdots,k\}^{d_j}\)} 
    \State \(F_{j|s}[\cdot,\boldsymbol{i}]\gets\) c.d.f. of \(X_j|\boldsymbol{X}_{\boldsymbol{p}(j)}=\boldsymbol{g}_{\boldsymbol{i}|s},S=s\)
    \State \(Q_{j|s}[\cdot,\boldsymbol{i}]\gets\) quantile of \(X_j|\boldsymbol{X}_{\boldsymbol{p}(j)}=\boldsymbol{g}_{\boldsymbol{i}|s},S=s\)
\EndFor
\EndFor
\end{algorithmic}
\end{algorithm}

\begin{figure}
    \centering
    \includegraphics[width=0.5\columnwidth]{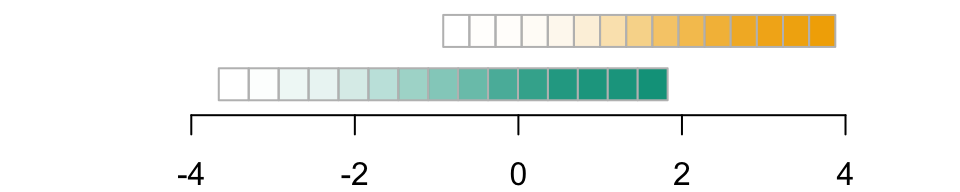}
    \caption{Visualization of vectors \(F_{1|0}\) (below) and \(F_{1|1}\) (on top) for the example of Figure~\ref{fig:gauss:algo}, with \(k=15\).}
    \label{fig:tensor:x1}
\end{figure}
\begin{figure}
\includegraphics[width=0.5\columnwidth]{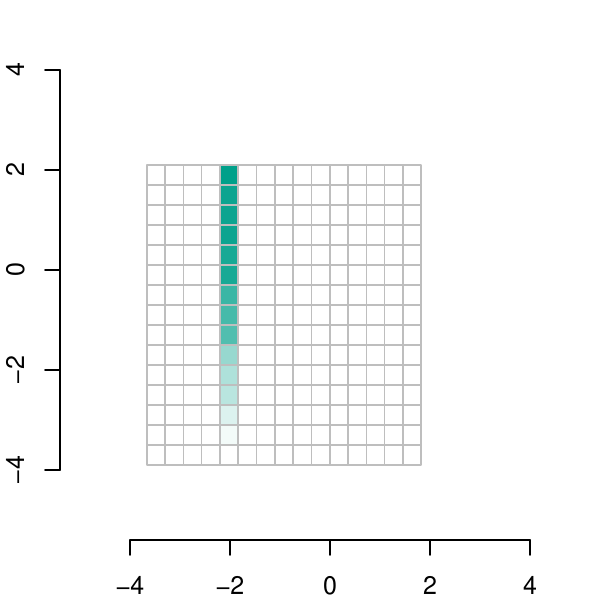}\includegraphics[width=0.5\columnwidth]{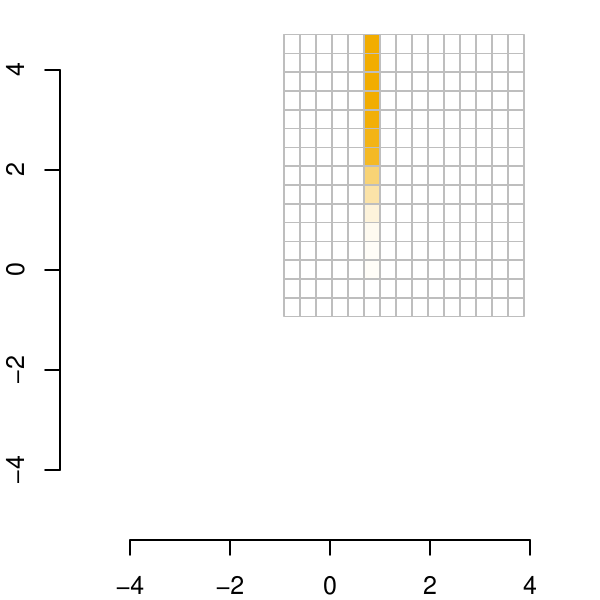}
    \caption{Visualization of matrices \(F_{2|0}\) (on the left) and \(F_{2|1}\) (on the right) for the example of Figure~\ref{fig:gauss:algo}, with \(k=15\). Vertical vectors are \(F_{2|0}[\cdot,i]\) and \(F_{2|1}[\cdot,j]\).}
    \label{fig:tensor:x2}
\end{figure}

\begin{algorithm}
\caption{Counterfactual calculation on causal graph (2)}\label{alg:3}
\begin{algorithmic}
\Require \(F_{1|s},\cdots,F_{d|s}\) and \(Q_{d|s},\cdots,Q_{d|s}\) (algorithm~\ref{alg:seqt-grid})
\Require grids \(\boldsymbol{g}_{1|s},\cdots,\boldsymbol{g}_{d|s}\), \(\mathcal{G}_{1|s},\cdots,\mathcal{G}_{d|s}\) and \(\boldsymbol{u}\)
\Require features \(\boldsymbol{a}\in\mathbb{R}^d\) (group 0)
\State \(\boldsymbol{b}\gets \boldsymbol{a}\)
\For{\(j\in \boldsymbol{v}\)} 
\State \(\boldsymbol{i}_0\gets \boldsymbol{a}_{\boldsymbol{p}(j)}\) on grid \(\mathcal{G}_{j|0}\) 
\State \(k_0\gets a_j\) on grid \(\boldsymbol{g}_{j|0}\)
\State \(p\gets F_{j|0}[k_0,\boldsymbol{i}_0]\)
\State \(\boldsymbol{i}_1\gets \boldsymbol{b}_{\boldsymbol{p}(j)}\) on grid \(\mathcal{G}_{j|1}\) 
\State \(k_1\gets p\) on grid \(\boldsymbol{u}\)
\State \(b_j\gets Q_{j|1}[k_1,\boldsymbol{i}_1]\)
\EndFor\\
\Return \(\boldsymbol{b}\) (counterfactual in group 1)
\end{algorithmic}
\end{algorithm}

\begin{algorithm}
\caption{Weigthed ecdf and eqf}\label{alg:weighted-ecdf}
\begin{algorithmic}
\Require $n$ observations \(\boldsymbol{x}\) and weights \(\boldsymbol{w}\)
\State \(\boldsymbol{x}\gets\text{sorted }\boldsymbol{x} \), and \(\boldsymbol{w}\) accordingly (and set \(x_0=-\infty\))
\Require points $x$ or probability level $u$
\State \(\boldsymbol{\overline{w}}\gets\text{cumulated sum of  }\boldsymbol{w} \)~(\(\overline{w}_0=0\) and \(\overline{w}_n=1\))
\State \(\hat{F}[x;\boldsymbol{x};\boldsymbol{w}]\gets\displaystyle{\overline{w}_{j-1}}\) where \(j\) satisfies \(\displaystyle{x_{j-1}\leq x<x_{j}}\)
\State \(\hat{Q}[u;\boldsymbol{x};\boldsymbol{w}]\gets x_j\) where \(j\) satisfies \(\displaystyle{\overline{w}_{j-1}\leq u<\overline{w}_{j}}\)\\
\Return \(\hat{F}[x;\boldsymbol{x};\boldsymbol{w}]\) (ecdf) and \(\hat{Q}[u;\boldsymbol{x};\boldsymbol{w}]\) (eqf)
\end{algorithmic}
\end{algorithm}

Even though Algorithm \ref{alg:seqt-grid} allows for calculating counterfactuals for a new observation without recalculating distribution quantities, it becomes complex as the number of parents of variables increases, resulting in a grid of dimension $k^{d+1}$ for a node with $d$ parents. To address this, we propose a novel algorithm, Algorithm \ref{alg:4}, to compute counterfactuals using sequential transport, while still leveraging the quantities calculated during the training step for a new observation. It is based on Algorithm \ref{alg:weighted-ecdf}, which provides simple codes to estimate the empirical CDF and the empirical quantile function when observations are weighted. More precise estimates are obtained using \citet{harrell2019package}'s \texttt{HMisc} R package and functions \texttt{wtd.stats} (inspired by \citet{harrell1982new}). Algorithm \ref{alg:4} is similar to Algorithm \ref{alg:1}, but weights are now calculated using distance metrics rather than Gaussian kernels.}

Following Section~\ref{sec:seq:transport}, the existence and uniqueness of sequential transport maps are guaranteed when the source and target conditional distributions are atomless. However, in many practical scenarios, causal graphs include categorical variables (e.g., Figures~\ref{fig:DAG-ADULT} and \ref{fig:DAG-COMPAS}), which leads to the loss of these theoretical guarantees. To address this issue in practice, if \( x_i \) is a categorical variable, we propose an alternative approach. First, we fit a multinomial model to predict the probabilities of category membership for \( x_i \) in the group \( s = 1 \). Next, using this model, we predict category probabilities for \( x_i \) based on the transported parent characteristics from group \( s = 0 \). For each prediction, we obtain a probability vector representing the likelihood of belonging to each category. Finally, we randomly draw a category using these probabilities as weights, thereby determining the transported category for the node \( x_i \).

\begin{algorithm}
\caption{Sequential transport with weights}\label{alg:4}
\begin{algorithmic}
\Require graph \(\mathcal{G}\) on \((s,\boldsymbol{x})\), with adjacency matrix \(\boldsymbol{A}\)
\Require dataset \((s_i,\boldsymbol{x}_i)\) and one individual \((s=0,\boldsymbol{a})\)
% \Require bandwidths \(\boldsymbol{b}\) (possibly \(\boldsymbol{b}_{|0}\) and \(\boldsymbol{b}_{|1}\))
\State \((s,\boldsymbol{v})\gets\boldsymbol{A}\) the topological ordering of vertices (DFS)
% \State \(T_s\gets\text{identity}\)
\State \(\boldsymbol{a}^\star\gets \boldsymbol{a}\)
\For{\(j\in \boldsymbol{v}\)} 
    \State \(\boldsymbol{p}(j) \gets \text{parents}(j)\)
    \State \((x_{i,j|s},\boldsymbol{x}_{i,\boldsymbol{p}(j)|s})\gets\) subsets when \(s\in\{0,1\}\)
    \State \(\boldsymbol{w}_{j|0}\gets 1/\text{dist}(\boldsymbol{x}_{\boldsymbol{p}(j)|0};\boldsymbol{a}_{\boldsymbol{p}(j)})\) 
    \State \(\boldsymbol{w}_{j|1}\gets 1/\text{dist}(\boldsymbol{x}_{\boldsymbol{p}(j)|1};\boldsymbol{a}_{\boldsymbol{p}(j)}^\star)\) 
    \State \(\boldsymbol{a}_j^\star\gets\hat{Q}[\hat{F}[\boldsymbol{a}_{j};\boldsymbol{x}_{j|0};\boldsymbol{w}_{j|0}];\boldsymbol{x}_{j|1};\boldsymbol{w}_{j|1}]\) (from Alg. \ref{alg:weighted-ecdf})
\EndFor\\
\Return \((s=1,\boldsymbol{a}^\star)\), counterfactual of \((s=0,\boldsymbol{a})\) 
\end{algorithmic}
\end{algorithm}

\section{Fairness Metrics} \label{appendix:fairness-metrics}

In this section, we describe the calculation of aggregated counterfactual fairness measures, extending commonly used group fairness metrics such as Equality of Opportunity (EqOp) \citep{hardt2016equality}, Class Balance (CB)  or False Negative Rate (FNR) \citep{Kleinberg_2018}, and Equal Treatment (EqTr) \citep{berk2021fairness}. To achieve this, the scoring classifier \( m(\cdot) \) is transformed into a threshold-based classifier \( m_t(\cdot) \), defined as \( m_t(\cdot) = 1 \) if \( m(\cdot) > t \), and \( m_t(\cdot) = 0 \) otherwise, with the threshold set at \( t = 0.5 \).

The Counterfactual Equality of Opportinities (CEqOp) is defined as
$$\text{CEqOp} := \text{TPR}_0^\star - \text{TPR}_0,$$
where $\text{TPR}_0^\star$ is the True Positive Rate (TPR) in the sample $\mathcal{D}_0$ (group with $s=0$)  when predictions are made using the counterfactuals $m_t(1, \boldsymbol{x}^\star)$, and $\text{TPR}_0$ is the TPR in $\mathcal{D}_0$ when predictions are based on the individuals' original values in $\mathcal{D}_0$. A positive value of $\text{CEqOp}$ indicates that the initial model is unfair towards the protected class.

The Counterfactual Class Balance (CCB or FNR)
$$\text{CCB(F)} := \frac{\text{TNR}_0^\star}{\text{TNR}_0}$$
where $\text{TNR}_0^\star$ is the True Negative Rate (TNR) of individuals in $\mathcal{D}_0$ calculated based on $m_t(1, \boldsymbol{x}^\star)$, and $\text{TNR}_0$ is the TNR of these individuals computed using $m_t(1, \boldsymbol{x})$.

Finally, Counterfactual Equal Treatment (CEqTr) corresponds to 
  $$\text{CEqTr} := \frac{\text{FPR}_0^\star}{\text{FNR}_0^\star} - \frac{\text{FPR}_0}{\text{FNR}_0}$$
where $\text{FPR}_0^\star$ and $\text{FNR}_0^\star$ are the False Positive Rate (FPR) and FNR computed based on the counterfactuals in the protected group, and $\text{FPR}_0$ and $\text{FNR}_0$ are their counterparts computed using the factual values for the same individuals.

\section{Additional Applications on Real Data} \label{appendix:add-real-data}

The application exercise from Section~\ref{sec:real-data} is replicated here using two more complex datasets: \texttt{adult income} and \texttt{COMPAS}. These datasets include a greater number of variables as well as a mix of numerical and categorical variables. We use cleaned version of these datasets available in the \texttt{fairadapt} R package \cite{plevcko2020fair}.

The first dataset is the \texttt{adult income} dataset, available from the UCI Machine Learning Repository \cite{misc_adult_2}. We follow the causal graph proposed by \citet{plevcko2021fairadapt}, reproduced in Figure~\ref{fig:DAG-ADULT}. The binary target variable $y$ indicates whether annual income exceeds \$50k, while the sensitive attribute $s$ represents the gender of employees (protected group: \(s = \text{``female''}\); other group: $s = \text{``male''}$). The other variables, $\boldsymbol{x}$, provide information on age (numerical), country of birth (categorical), marital status (categorical), education level (numerical), work class (categorical), weekly working hours (numerical), and occupation (categorical).

\begin{figure}[htb]
    \centering
    \resizebox{.6\columnwidth}{!}{
\tikz{
    \node[fill=red!30] (n1) at (3.8, 0.0) {sex};
\node[fill=yellow!60] (n2) at (7.5, 0.1) {age};
\node[fill=yellow!60] (n3) at (5.1, 5.0) {native country};
\node[fill=yellow!60] (n4) at (3.3, 3.1) {marital status};
\node[fill=yellow!60] (n5) at (6.4, 2.8) {education num};
\node[fill=yellow!60] (n6) at (9.5, 3.4) {workclass};
\node[fill=yellow!60] (n7) at (10.0, 1.6) {hours per week};
\node[fill=yellow!60] (n8) at (0.0, 1.9) {occupation};
\node[fill=blue!30] (n9) at (4.5, 1.6) {income};
\draw[->, black] (n1) -- (n4);
\draw[->, black] (n1) -- (n5);
\draw[->, black] (n1) -- (n6);
\draw[->, black] (n1) -- (n7);
\draw[->, black] (n1) -- (n8);
\draw[->, black] (n1) -- (n9);
\draw[->, black] (n2) -- (n4);
\draw[->, black] (n2) -- (n5);
\draw[->, black] (n2) -- (n6);
\draw[->, black] (n2) -- (n7);
\draw[->, black] (n2) -- (n8);
\draw[->, black] (n2) -- (n9);
\draw[->, black] (n3) -- (n4);
\draw[->, black] (n3) -- (n5);
\draw[->, black] (n3) -- (n6);
\draw[->, black] (n3) -- (n7);
\draw[->, black] (n3) -- (n8);
\draw[->, black] (n3) -- (n9);
\draw[->, black] (n4) -- (n5);
\draw[->, black] (n4) -- (n6);
\draw[->, black] (n4) -- (n7);
\draw[->, black] (n4) -- (n8);
\draw[->, black] (n4) -- (n9);
\draw[->, black] (n5) -- (n6);
\draw[->, black] (n5) -- (n7);
\draw[->, black] (n5) -- (n8);
\draw[->, black] (n5) -- (n9);
\draw[->, black] (n6) -- (n9);
\draw[->, black] (n7) -- (n9);
\draw[->, black] (n8) -- (n9);
}  
}
    \caption{Causal graph of the \texttt{adult} dataset.}\label{fig:DAG-ADULT}
\end{figure}

The second dataset, \texttt{COMPAS} \cite{Larson2016}, contains individual-level information used to predict whether a criminal defendant is likely to reoffend within two years ($y$). Again, we follow the causal graph proposed by \citet{plevcko2021fairadapt}, reproduced in Figure~\ref{fig:DAG-COMPAS}. The sensitive attribute is race (protected value: $s = \text{``Non-White''}$; other value: $s = \text{``White''}$).  The individual features $\boldsymbol{x}$ include age (numerical), gender (binary), the number of juvenile felonies (numerical), juvenile misdemeanors (numerical), other juvenile offenses (numerical), prior offenses (numerical), and the degree of charge (binary).

\begin{figure}[htb]
    \centering
    \resizebox{.6\columnwidth}{!}{
\tikz{
\node[fill=yellow!60] (n1) at (4.8, 5.0) {age};
\node[fill=yellow!60] (n2) at (2.5, 0.1) {sex};
\node[fill=yellow!60] (n3) at (10.0, 1.9) {juv fel count};
\node[fill=yellow!60] (n4) at (0.4, 3.4) {juv misd count};
\node[fill=yellow!60] (n5) at (0.0, 1.6) {juv other count};
\node[fill=yellow!60] (n6) at (3.4, 2.8) {priors count};
\node[fill=yellow!60] (n7) at (6.7, 3.2) {c charge degree};
\node[fill=red!30] (n8) at (6.2, 0.0) {race};
\node[fill=blue!30] (n9) at (5.5, 1.6) {two year recid};
\draw[->, black] (n1) -- (n3);
\draw[->, black] (n1) -- (n4);
\draw[->, black] (n1) -- (n5);
\draw[->, black] (n1) -- (n6);
\draw[->, black] (n1) -- (n7);
\draw[->, black] (n1) -- (n9);
\draw[->, black] (n2) -- (n3);
\draw[->, black] (n2) -- (n4);
\draw[->, black] (n2) -- (n5);
\draw[->, black] (n2) -- (n6);
\draw[->, black] (n2) -- (n7);
\draw[->, black] (n2) -- (n9);
\draw[->, black] (n3) -- (n6);
\draw[->, black] (n3) -- (n7);
\draw[->, black] (n3) -- (n9);
\draw[->, black] (n4) -- (n6);
\draw[->, black] (n4) -- (n7);
\draw[->, black] (n4) -- (n9);
\draw[->, black] (n5) -- (n6);
\draw[->, black] (n5) -- (n7);
\draw[->, black] (n5) -- (n9);
\draw[->, black] (n6) -- (n7);
\draw[->, black] (n6) -- (n9);
\draw[->, black] (n7) -- (n9);
\draw[->, black] (n8) -- (n3);
\draw[->, black] (n8) -- (n4);
\draw[->, black] (n8) -- (n5);
\draw[->, black] (n8) -- (n6);
\draw[->, black] (n8) -- (n7);
\draw[->, black] (n8) -- (n9);
}
}
    \caption{Causal graph of the \texttt{COMPAS} dataset.}\label{fig:DAG-COMPAS}
\end{figure}

For each dataset, we train a logistic regression model to predict the binary target variable $y$ using the features $\boldsymbol{x}$. Two versions of the model are considered: an ``aware'' model, where the sensitive attribute $s$ is included as a feature, and an ``unaware'' model, where the model is blind to the sensitive attribute. These models are then employed to compare predictions based on the original features (factuals) and the counterfactual features. As in Section~\ref{sec:real-data}, we consider four types of counterfactuals: (i) naive (or \textit{ceteris paribus}), where only the sensitive attribute for the protected group is changed ($s \leftarrow 1$); (ii) multivariate OT, where the features of the protected group are transformed using multivariate OT; (iii) fairadapt; and (iv) sequential transport based on the causal graphs shown in Figures~\ref{fig:DAG-ADULT} and \ref{fig:DAG-COMPAS}. 

The score distributions are plotted in Figure~\ref{fig:densities-adult-compas-pred} for the \texttt{adult income} dataset (left) and the \texttt{COMPAS} dataset (right), for the ``aware'' model only (the results of the ``unaware'' model based on different counterfactuals for the protected group are provided in the Online Replication Ebook). The green curves represent the estimated density of scores predicted by the ``aware'' model using the factual features of individuals from the protected group. The dashed gold curves show the score distributions of the same model for the reference group. All other curves represent the estimated densities of scores predicted by the ``aware'' model based on the different counterfactuals for the protected group.

As observed with the \texttt{law school} dataset in the main part of the paper, the score distributions obtained from counterfactuals of the protected group using our sequential transport approach closely align with those obtained using fairadapt for the \texttt{adult income} dataset. The results are slightly more nuanced for the \texttt{COMPAS} dataset. Nevertheless, for both additional examples, the score distributions predicted by the ``aware'' model for counterfactual individuals approach the score distribution of the reference group. This stands in contrast to the distinct distribution observed for the protected group when counterfactuals are not considered.

\begin{figure}[htb]
    \centering
    \includegraphics[width=.5\columnwidth]{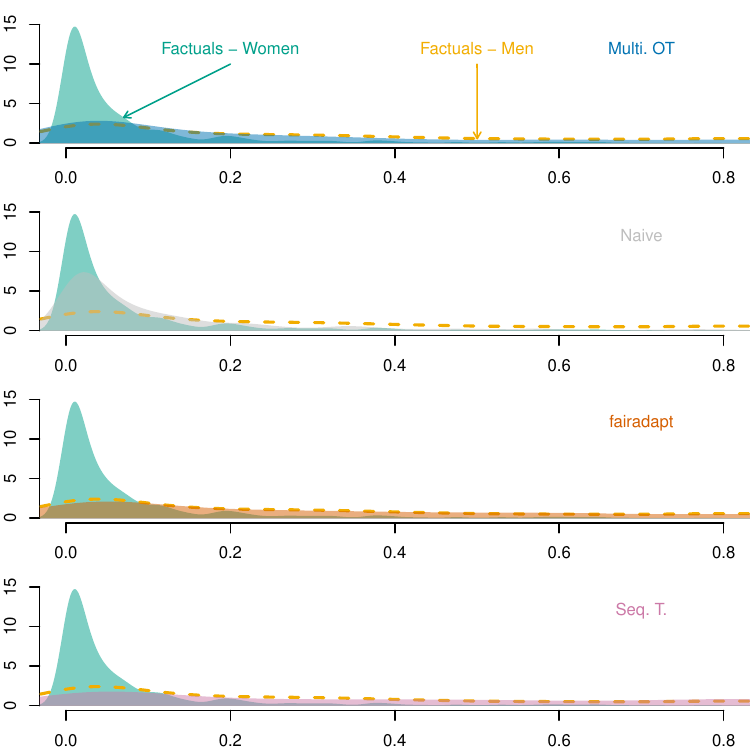}\includegraphics[width=.5\columnwidth]{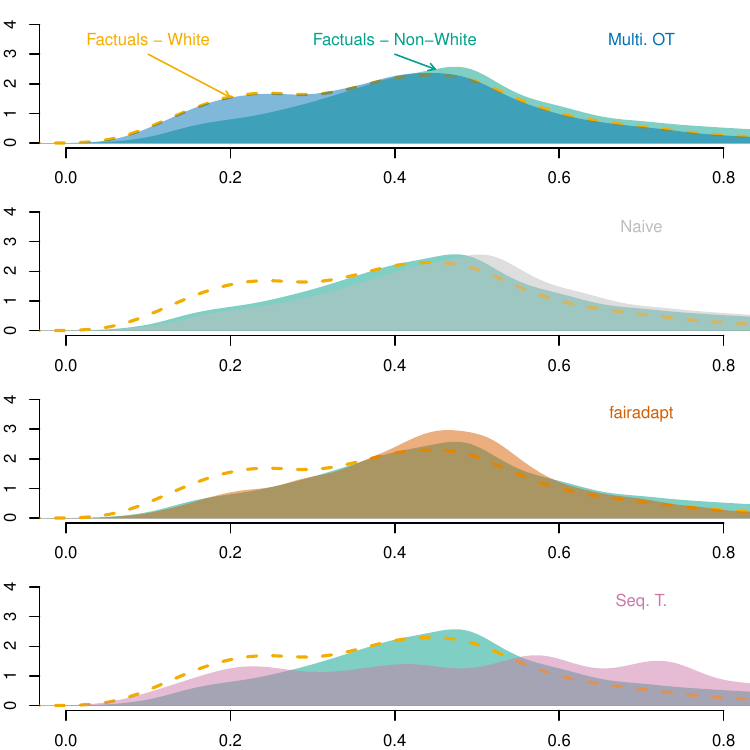}
    \caption{Densities of predicted scores (aware model) for all individuals from the minority class (Women on the left, Black individuals on the right) with factuals and counterfactuals, for the \texttt{adult} dataset (on the left) and the \texttt{COMPAS} datset (on the right). The dashed line represents the density of predicted scores for the observed individuals from the majority group.}
    \label{fig:densities-adult-compas-pred}
\end{figure}

Table~\ref{tab:fairness-metrics-adult-compas} presents various metrics computed from the scores of both ``aware'' and ``unaware'' models, using either the initial observations or the counterfactuals, for these datasets. For each model, the first two rows report the true positive rate (TPR) and false positive rate (FPR) within the groups: protected group (either Women or Non-White, depending on the dataset), using different counterfactual constructions, and other group (Men or White) in the final column. The scoring classifier $m(\cdot)$ is transformed into a threshold-based classifier $m_t(\cdot)$, where $m_t(\cdot)=1$ if $m(\cdot) > t$, and $m_t(\cdot)=0$ otherwise. The remaining rows provide the counterfactual fairness metrics presented in Appendix~\ref{appendix:fairness-metrics}, computed exclusively for individuals in group $\mathcal{D}_0$.

{
\setlength{\tabcolsep}{.65mm}
\begin{table}
\centering\small
\begin{tabular}{lrrrrrr}
\toprule
& \multicolumn{6}{c}{\textit{Adult dataset}}\\
\cmidrule{2-7}
& Women & Naive & OT & Fairadapt & Seq & Men\\
 & $s=0$ & $s \leftarrow 1$ & $s \leftarrow 1$ & $s \leftarrow 1$ & $s \leftarrow 1$ & $s=1$\\
 \addlinespace[0.3em]
 No. Obs. & 662 & 662 & 662 & 662 & 662 & 1,338\\
\cmidrule{2-7}
\addlinespace[0.3em]
\multicolumn{7}{l}{\textbf{Aware model}}\\
\hspace{.5em}TPR & 0.26 & 0.44 & 0.67 & 0.74 & 0.61 & 0.54\\
\hspace{.5em}FPR & 0.01 & 0.02 & 0.15 & 0.23 & 0.32 & 0.10\\
& & \multicolumn{4}{c}{Individuals in $\mathcal{D}_0$}\\
\cmidrule(lr){3-6}
\hspace{.5em}CDP &  & 0.05 & 0.17 & 0.24 & 0.29 & \\
\hspace{.5em}CEqOp &  & 0.19 & 0.41 & 0.48 & 0.35 & \\
\hspace{.5em}CCB(FNR) &  & 0.75 & 0.45 & 0.35 & 0.52 & \\
\hspace{.5em}CEqTr &  & 0.04 & 0.44 & 0.86 & 0.81 & \\
\addlinespace[0.3em]
\multicolumn{7}{l}{\textbf{Unaware model}}\\
\hspace{.5em}TPR & 0.41 & 0.41 & 0.63 & 0.74 & 0.56 & 0.54\\
\hspace{.5em}FPR & 0.02 & 0.02 & 0.14 & 0.21 & 0.31 & 0.10\\
& & \multicolumn{4}{c}{Individuals in $\mathcal{D}_0$}\\
\cmidrule(lr){3-6}
\hspace{.5em}CDP &  & 0.00 & 0.12 & 0.20 & 0.25 & \\
\hspace{.5em}CEqOp &  & 0.00 & 0.22 & 0.33 & 0.15 & \\
\hspace{.5em}CCB(FNR) &  & 1.00 & 0.62 & 0.44 & 0.75 & \\
\hspace{.5em}CEqTr &  & 0.00 & 0.35 & 0.80 & 0.67 & \\
%
% COMPAS
%
\midrule
&\multicolumn{5}{c}{\textit{COMPAS dataset}}\\
\cmidrule{2-7}
\multicolumn{2}{r}{Non-White} & Naive & OT & Fairadapt & Seq & White\\
 & $s=0$ & $\leftarrow 1$ & $\leftarrow 1$ & $s \leftarrow 1$ & $s \leftarrow 1$ & $s=1$\\
 \addlinespace[0.3em]
No. Obs. & 4,760 & 4,760 & 4,760 & 4,760 & 4,760& 24,2454\\
\cmidrule{2-7}
\multicolumn{7}{l}{\textbf{Aware model}}\\
\hspace{.5em}TPR & 0.57 & 0.67 & 0.39 & 0.53 & 0.65 & 0.45\\
\hspace{.5em}FPR & 0.23 & 0.31 & 0.18 & 0.22 & 0.41 & 0.19\\
& & \multicolumn{4}{c}{Individuals in $\mathcal{D}_0$}\\
\cmidrule(lr){3-6}
\hspace{.5em}CDP &  & 0.02 & -0.07 & -0.02 & 0.02 & \\
\hspace{.5em}CEqOp &  & 0.10 & -0.18 & -0.04 & 0.08 & \\
\hspace{.5em}CCB(FNR) &  & 0.76 & 1.42 & 1.10 & 0.82 & \\
\hspace{.5em}CEqTr &  & 0.41 & -0.23 & -0.07 & 0.62 & \\
\addlinespace[0.3em]
\multicolumn{6}{l}{\textbf{Unaware model}}\\
\hspace{.5em}TPR & 0.61 & 0.61 & 0.35 & 0.47 & 0.62 & 0.40\\
\hspace{.5em}FPR & 0.25 & 0.25 & 0.14 & 0.17 & 0.38 & 0.15\\
%\hspace{.5em}n\_obs & 4760.00 & 4760.00 & 4760.00 & 4760.00 & 2454.00\\
& & \multicolumn{4}{c}{Individuals in $\mathcal{D}_0$}\\
\cmidrule(lr){3-6}
\hspace{.5em}CDP &  & 0.00 & -0.09 & -0.05 & 0.00 & \\
\hspace{.5em}CEqOp &  & 0.00 & -0.26 & -0.14 & 0.02 & \\
\hspace{.5em}CCB(FNR) &  & 1.00 & 1.66 & 1.36 & 0.96 & \\
\hspace{.5em}CEqTr &  & 0.00 & -0.42 & -0.33 & 0.36 & \\
\bottomrule
\end{tabular}
\caption{Fairness metrics for the \texttt{adult} dataset (top) and the \texttt{COMPAS} dataset (bottom), comparing classifier predictions based on original features $(s, \boldsymbol{x})$ and counterfactuals $(s=1, \boldsymbol{x})$, constructed using different techniques: naive, OT, Fairadapt, and sequential transport. For metrics computed exclusively on individuals in \(\mathcal{D}_0\) (Women or Non-White), the values obtained using counterfactuals are compared to those obtained using factuals.}
    \label{tab:fairness-metrics-adult-compas}
\end{table}
}

Table~\ref{tab:fairness-metrics-law} provides the same metrics for the \texttt{law school} dataset, to complement Table~\ref{tab:dpcounterfactual}.

{
\setlength\tabcolsep{.8mm}
\begin{table}
\centering\small
\begin{tabular}{lrrrrrr}
\toprule
 & Black & Naive & OT & Fairadapt & Seq & White\\
 & $s=0$ & $s \leftarrow 1$ & $s \leftarrow 1$ & $s \leftarrow 1$ & $s \leftarrow 1$ & $s=1$\\
 \addlinespace[0.3em]
No. Obs. & 1,282 & 1,282 & 1,282 & 1,282 & 1,282 & \vphantom{1} 18,285\\
\cmidrule(lr){2-7}
\addlinespace[0.3em]
\multicolumn{7}{l}{\textbf{Aware model}}\\
\hspace{.5em}TPR & 0.00 & 0.15 & 0.64 & 0.66 & 0.68 & 0.65\\
\hspace{.5em}FPR & 0.00 & 0.08 & 0.57 & 0.63 & 0.64 & 0.51\\
& & \multicolumn{4}{c}{Individuals in $\mathcal{D}_0$}\\
\cmidrule(lr){3-6}
\hspace{.5em}CDP &  & 0.22 & 0.37 & 0.38 & 0.39 & \\
\hspace{.5em}CEqOp &  & 0.15 & 0.64 & 0.66 & 0.68 & \\
\hspace{.5em}CCB(FNR) &  & 0.85 & 0.36 & 0.34 & 0.32 & \\
\hspace{.5em}CEqTr &  & 0.10 & 1.59 & 1.86 & 2.04 & \\
\multicolumn{7}{l}{\textbf{Unaware model}}\\
\hspace{.5em}TPR & 0.11 & 0.11 & 0.60 & 0.62 & 0.62 & 0.60\\
\hspace{.5em}FPR & 0.07 & 0.07 & 0.52 & 0.56 & 0.56 & 0.45\\
& & \multicolumn{4}{c}{Individuals in $\mathcal{D}_0$}\\
\cmidrule(lr){3-6}
\hspace{.5em}CDP &  & 0.00 & 0.18 & 0.19 & 0.20 & \\
\hspace{.5em}CEqOp &  & 0.00 & 0.49 & 0.51 & 0.51 & \\
\hspace{.5em}CCB(FNR) &  & 1.00 & 0.45 & 0.42 & 0.43 & \\
\hspace{.5em}CEqTr &  & 0.00 & 1.22 & 1.41 & 1.40 & \\
\bottomrule
\end{tabular}
\caption{Fairness metrics for the \texttt{law school} dataset, comparing classifier predictions based on original features $(s, \boldsymbol{x})$ and counterfactuals $(s=1, \boldsymbol{x})$, constructed using different techniques: naive, OT, Fairadapt, and sequential transport. For metrics computed exclusively on individuals in \(\mathcal{D}_0\) (Black individuals), the values obtained using counterfactuals are compared to those obtained using factuals.}
    \label{tab:fairness-metrics-law}
\end{table}
}

\section{Wrong Causal Assumptions} \label{appendix:wrong-assumptions}

In this section, we illustrate the impact of incorrect causal assumptions between two independent variables through an example with a small sample size of $n=500$. Similar to the approach in Section~\ref{sec:fairness}, we simulate a dataset $\{(y_i, s_i, \boldsymbol{x}_i)\}_{i=1}^{n}$ composed of two legitimate variables $X_1$ and $X_2$, an outcome $Y$, and a sensitive attribute $S$. The variables $X_1$ and $X_2$ are uniformly distributed and independent of each other. However, they depend on the sensitive attribute, as the distribution parameters vary according to $S$. Specifically:  
$$
\begin{cases}
X_1 \sim \mathcal{U}(0,1), \quad X_2 \sim \mathcal{U}(0,1) &\text{if } S=0,\\
X_1 \sim \mathcal{U}(1,2), \quad X_2 \sim \mathcal{U}(1,2) &\text{if } S=1.
\end{cases}
$$

The outcome values are drawn from a Bernoulli distribution $Y_i \sim \mathcal{B}(p_i)$, where the success parameter $p_i \in [0,1]$ is individual-specific:
$$
p_i = \frac{\exp(\eta_i)}{1 + \exp(\eta_i)},
$$
with
$$
\eta_i = 
\begin{cases}
0.6 x_1 + 0.2 x_2, & \text{if } s_i=0,\\
0.4 x_1 + 0.3 x_2, & \text{if } s_i=1.
\end{cases}
$$

The causal graph corresponding to this simulated data is shown in Figure~\ref{fig:DAG-wrong-assumptions}(a). We investigate the impact of incorrect causal assumptions on sequential transport estimates by assuming that \(X_1\) causes \(X_2\) (Figure~\ref{fig:DAG-wrong-assumptions}(b)) or, alternatively, that \(X_2\) causes \(X_1\) (Figure~\ref{fig:DAG-wrong-assumptions}(c)).

\begin{figure}[htb]
    \centering
    \begin{tabular}{ccccc}
    \tikz{
        \useasboundingbox (0, -0.5) rectangle (2, 1);
        \node[fill=wongPurple!30] (s) at (0,0) {$S$};
        \node[fill=wongPurple!30] (x2) at (1,-.5) {$X_{2}$};
        \node[fill=wongPurple!30] (x1) at (1,.5) {$X_{1}$};
        \node[fill=wongPurple!30] (y) at (2,0) {$Y$};
        \node[] (a) at (1,1) {(a)};
        \path[->, black] (s) edge (x1);
        \path[->, black] (s) edge (x2);
        \path[->, black] (x1) edge (y);
        \path[->, black] (x2) edge (y);
        \path[->, black, bend right=80] (s) edge (y);
    } & & \tikz{
        \useasboundingbox (0, -0.5) rectangle (2, 1);
        \node[fill=gris] (s) at (0,0) {$S$};
        \node[fill=gris] (x2) at (1,-.5) {$X_{2}$};
        \node[fill=gris] (x1) at (1,.5) {$X_{1}$};
        \node[fill=gris] (y) at (2,0) {$Y$};
        \node[] (a) at (1,1) {(b)};
        \path[->, black] (s) edge (x1);
        \path[->, red, very thick] (x1) edge (x2);
        \path[->, black] (s) edge (x2);
        \path[->, black] (x1) edge (y);
        \path[->, black] (x2) edge (y);
        \path[->, black, bend right=80] (s) edge (y);
    } & & \tikz{
        \useasboundingbox (0, -0.5) rectangle (2, 1);
        \node[fill=wongLightBlue] (s) at (0,0) {$S$};
        \node[fill=wongLightBlue] (x2) at (1,-.5) {$X_{2}$};
        \node[fill=wongLightBlue] (x1) at (1,.5) {$X_{1}$};
        \node[fill=wongLightBlue] (y) at (2,0) {$Y$};
        \node[] (a) at (1,1) {(c)};
        \path[->, black] (s) edge (x1);
        \path[->, red, very thick] (x2) edge (x1);
        \path[->, black] (s) edge (x2);
        \path[->, black] (x1) edge (y);
        \path[->, black] (x2) edge (y);
        \path[->, black, bend right=80] (s) edge (y);
    } 
\end{tabular}
    \caption{Causal assumptions on simulated data, with a correct assumption on the left and two wrong assumptions in the middle (where $X_1$ is assumed to cause $X_2$) and on the right (where $X_2$ is assumed to cause $X_1$).}
    \label{fig:DAG-wrong-assumptions}
\end{figure}

We consider a hypothetical scoring model $m(\cdot)$, a logistic regression, which estimates the outcome based on the two covariates and the sensitive. Specifically:
\[
m(x_1,x_2,s)=\big(1+\exp\big[-\big((x_1+x_2)/2 + \boldsymbol{1}(s=1)\big)\big]\big)^{-1}.
\]

The iso-curves of this scoring classifier are shown in Figure~\ref{fig:wrong-assumptions-individual} for $s=0$ (left) and $s=1$ (right). This figure depicts the individual $(s = 0, x_1 = 0.5, x_2 = 0.5)$, predicted at 62.2\% by the model $m(\cdot)$. When only the sensitive attribute is changed to $s=1$, the naive model predicts a value of 81.8\%. Under the correct causal assumption (purple point), the counterfactual values $(s = 1, x_1^\star, x_2^\star)$ are close to those obtained using multivariate OT. The model then predicts a value of 92.5\% with sequential transport under the correct causal assumption, which is close to the 90.4\% predicted using the counterfactual constructed via OT. 

When an incorrect causal assumption is made (gray point for assuming $X_1$ causes $X_2$, and light blue point for assuming $X_2$ causes $X_1$), the counterfactual values remain very close to those obtained with a correct assumption on the causal structure.

\begin{figure}[htb]
    \centering
    \includegraphics[width=.5\columnwidth]{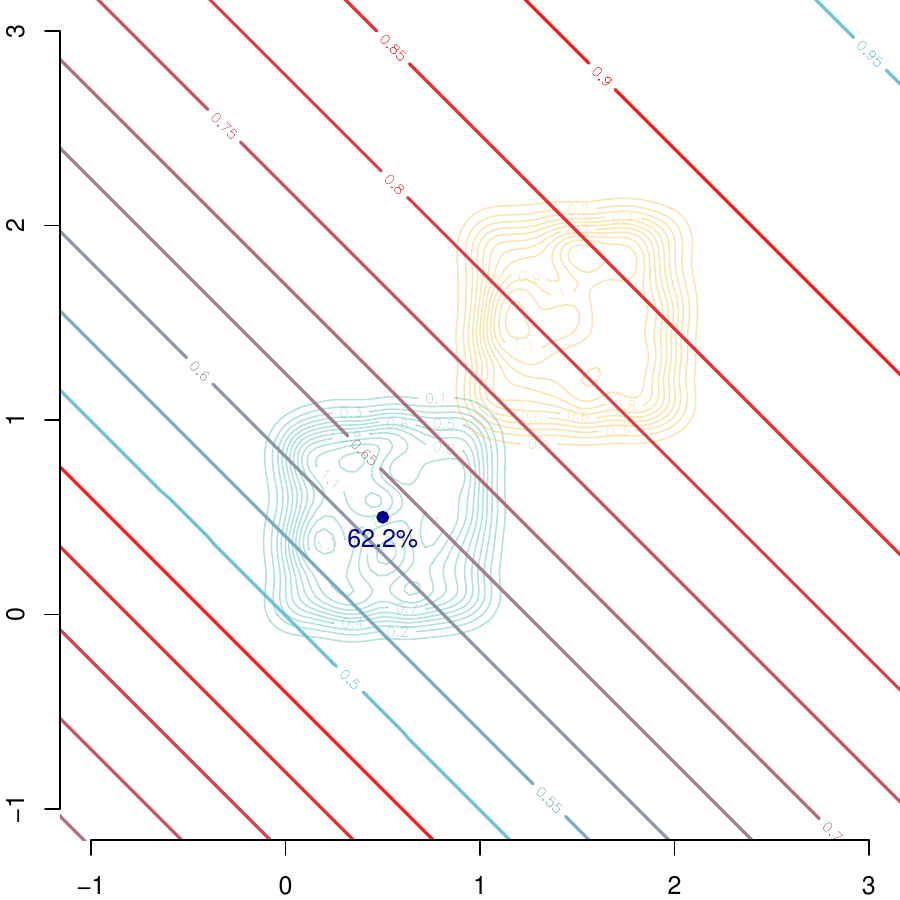}\includegraphics[width=.5\columnwidth]{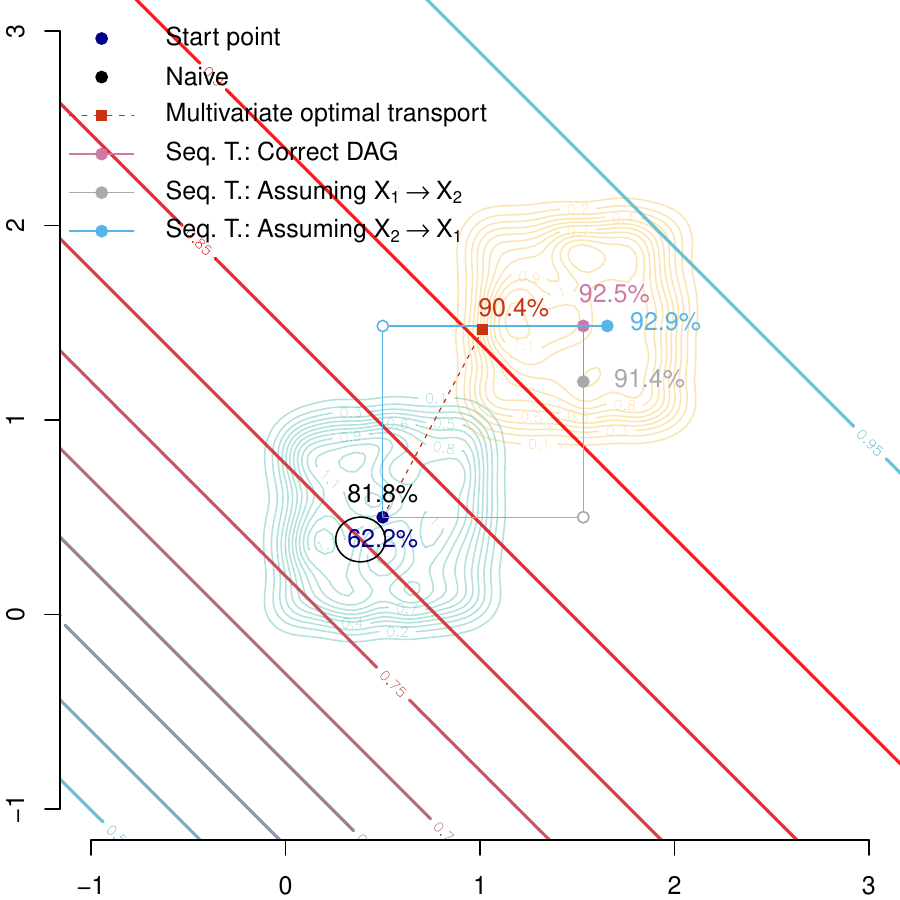}
    \caption{The iso-curves for $m(0, x_1, x_2)$ (left) and $m(1, x_1, x_2)$ (right) are shown in the background. The blue dot represents the individual $(s = 0, x_1 = 0.5, x_2 = 0.5)$, predicted at 62.2\% by the model $m(\cdot)$. On the right, the purple dot corresponds to the counterfactual $(s = 1, x_1^\star, x_2^\star)$ obtained using sequential transport under the correct causal graph from Figure~\ref{fig:DAG-wrong-assumptions}(a). Counterfactuals derived under incorrect assumptions (Figure~\ref{fig:DAG-wrong-assumptions}(b) and Figure~\ref{fig:DAG-wrong-assumptions}(c)) are depicted by the gray and blue dots, respectively. The red square represents the counterfactual obtained using multivariate OT.}
    \label{fig:wrong-assumptions-individual}
\end{figure}

To gain a better understanding of this example beyond the analysis of a single point, Figure~\ref{fig:wrong-assumptions-densities} presents the bivariate densities of the counterfactuals \((s=1, x_1^\star, x_2^\star)\) estimated using kernel density estimation. The densities are shown for counterfactuals obtained via OT (top left) and sequential transport, using the correct causal graph (top right) and incorrect causal assumptions (bottom). The estimated densities of the factual values are also displayed, in green for group \( s=0 \) and yellow for group \( s=1 \). The conclusions observed for the single point extend to the sample level. The density of the counterfactuals is very to the factual distribution either when the correct causal assumption is made or when this assumption is wrong.

\begin{figure}[htb]
    \centering
    \includegraphics[width=\columnwidth]{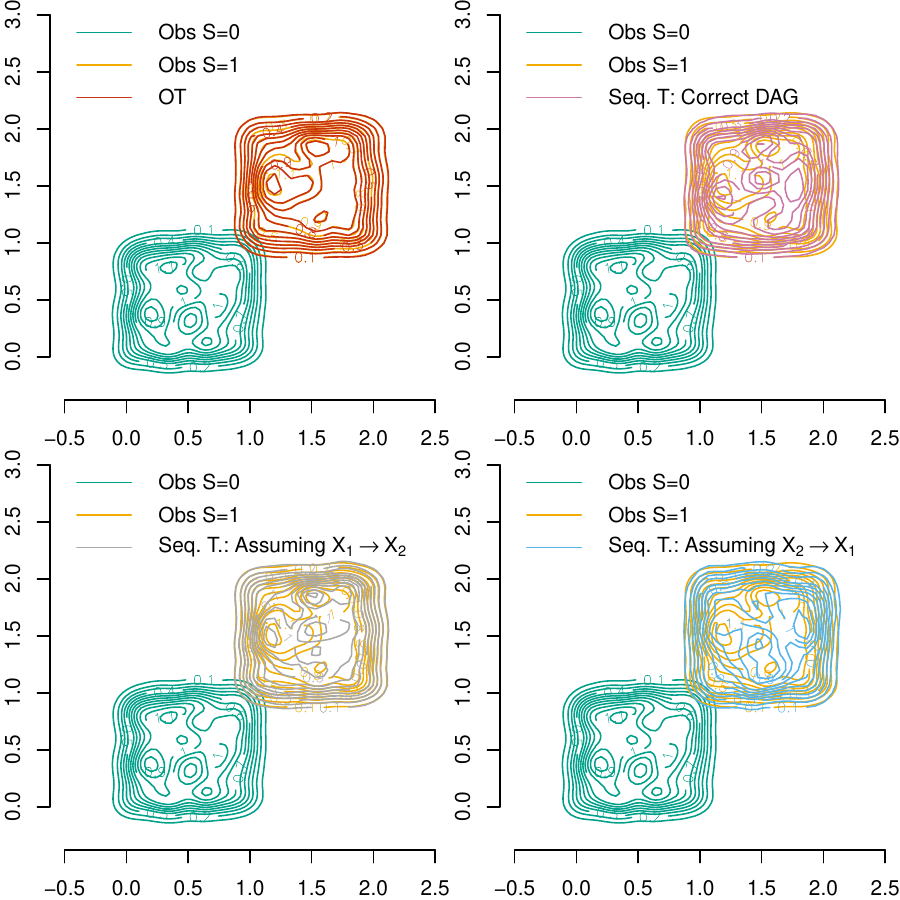}
    \caption{Estimated densities of the factuals in both groups and estimated densities using the counterfactuals either with optimal transport (top left), or sequential transport under a correct causal assumption (top right), a wrong assumption where $X_1$ causes $X_2$ (bottom left) and another wrong assumption where $X_2$ causes $X_1$ (bottom right).}
    \label{fig:wrong-assumptions-densities}
\end{figure}

\clearpage

\bibliography{biblio}

\end{document}